\documentclass[journal,review]{IEEEtran}
\usepackage{amsmath,amsfonts}
\usepackage{array}
\usepackage[caption=false,font=normalsize,labelfont=sf,textfont=sf]{subfig}
\usepackage{diagbox}
\usepackage{textcomp}
\usepackage{stfloats}
\usepackage{url}
\usepackage{verbatim}
\usepackage{graphicx}
\usepackage{amsmath}
\usepackage{multirow}
\usepackage[colorlinks=true, allcolors=blue]{hyperref}
\usepackage{array} % 表格居中
\usepackage{float}
\usepackage{filecontents}
\usepackage{booktabs}
\usepackage[table]{xcolor}
\usepackage{algorithm}
\usepackage{algpseudocode}
\usepackage{wrapfig} % 用于插入biography
\usepackage{flushend} % 用于让最后一页的biography以行对齐
\usepackage{mathtools}
\usepackage{graphicx,amsfonts,amsmath,amsthm,color,stfloats}
\usepackage[caption=false,font=normalsize,labelfont=sf,textfont=sf]{subfig}
\usepackage{placeins}

\usepackage{amssymb}
\usepackage{booktabs} 
\newtheorem{theorem}{Theorem}
\newtheorem{corollary}{Corollary}
\newtheorem{proposition}{Proposition}
\newtheorem{lemma}{Lemma}
\newtheorem*{remark}{Remark}
\newcommand{\rF}{\rm F}

\usepackage{algpseudocode} 
\newcommand{\argmin}{\mathop{\mathrm{argmin}}}
\newcommand{\argmax}{\mathop{\mathrm{argmax}}}

\def\BibTeX{{\rm B\kern-.05em{\sc i\kern-.025em b}\kern-.08em
    T\kern-.1667em\lower.7ex\hbox{E}\kern-.125emX}}

\begin{document}
\title{Regularized Optimization on Grassmann Manifold: Theory, Algorithm and Applications}
%Sparse  Projection Matrix Approximation 
%with Applications to Community Detection}
\author{Zhuan~Liang and Zheng~Zhai
\IEEEcompsocitemizethanks{
\IEEEcompsocthanksitem Zhuan Liang and Zheng Zhai are with Department of Statistics, Faculty of Arts and Sciences at Beijing Normal University, Zhuhai. This work was supported by the National Natural Science Foundation of China under Grant of 12301478.
\IEEEcompsocthanksitem Corresponding author: Zheng Zhai. E-mail:zhaizheng@bnu.edu.cn
}

}
\maketitle
%\begin{abstract}
\begin{abstract}
Spectral methods are among the most widely used techniques for community detection, clustering, and graph learning. Their performance, however, critically depends on the accurate estimation of the underlying spectral subspace and can deteriorate substantially in the presence of noise, outliers, or model perturbations. To address this limitation, we propose a Regularized Projection Matrix Approximation (RPMA) framework for robust estimation of rank-(K) projection matrices. RPMA extends classical spectral projection by incorporating a regularization term, producing projection estimates that are more robust, sparse, and interpretable.

We formulate the proposed model as an optimization problem on the manifold of rank-(K) projection matrices and exploit its geometric equivalence to the Grassmann manifold. Based on this manifold characterization, we derive the first- and second-order optimality conditions, establish the local stability of the regularized leading eigenspace, and characterize the stability of the critical-point landscape under sufficiently small regularization. To efficiently solve the resulting nonconvex optimization problem, we develop a Riemannian gradient projection algorithm with backtracking line search, together with a more efficient Cayley–Sherman–Morrison–Woodbury (Cayley–SMW) gradient method that avoids repeated eigendecompositions.

Extensive experiments on both synthetic and real-world datasets demonstrate that RPMA substantially improves the recovery accuracy of projection matrices and consistently outperforms conventional spectral projection methods for community detection and clustering under noisy environments.

%\end{abstract}
%This paper introduces a sparse regularized projection matrix approximation model for recovering cluster structures from affinity matrices. The model is formulated as a projection approximation problem with an entry-wise sparsity penalty to encourage sparse solutions. We propose two algorithms to solve this problem: one involves direct optimization on the Stiefel manifold using the Cayley transformation, while the other employs the Alternating Direction Method of Multipliers (ADMM). Numerical experiments on both synthetic and real-world datasets demonstrate that our regularized projection matrix approximation approach significantly outperforms state-of-the-art methods in clustering accuracy and performance.
\end{abstract}


\section{Introduction   }
%\IEEEPARstart{C}{ommunity} detection is a crucial problem in unsupervised learning that has garnered attention from researchers across various disciplines, including mathematics, statistics,  physics, and social sciences. The objective is to partition  $n$ data points into $K$ groups based on their pairwise similarities, represented by a similarity matrix $A\in {\mathbb R}^{n×n}$. A prevalent approach to tackle this problem involves first deriving a lower-dimensional representation of the data from $A$. Subsequently, a clustering algorithm such as $k$-means or the EM algorithm is applied to identify the clusters.  The success of this method depends on the quality of the data representation and the accuracy of the computational methods used for $A$.

\IEEEPARstart{C}{ommunity} detection is a fundamental problem in unsupervised learning that has attracted considerable attention from researchers in a wide range of disciplines, including mathematics, statistics, physics, and the social sciences. The goal is to partition \(n\) data points into \(K\) communities based on their pairwise similarities, which are typically represented by a similarity matrix \(A\in\mathbb{R}^{n\times n}\). A widely adopted strategy for addressing this problem is to first extract a low-dimensional representation from \(A\), and then apply a clustering algorithm, such as \(k\)-means or the expectation--maximization (EM) algorithm, to identify the underlying communities. The effectiveness of this approach critically depends on the quality of the extracted representation, which in turn relies on both the structural properties of the similarity matrix and the accuracy of the computational methods used to process it.

\begin{comment}
\IEEEPARstart{C}{ommunity} detection is a fundamental problem in unsupervised learning and has attracted extensive attention in statistics, machine learning, network science, and social sciences. Given a similarity or adjacency matrix \(A\in\mathbb{R}^{n\times n}\) that characterizes pairwise relationships among \(n\) objects, the objective is to partition the objects into \(K\) communities such that objects within the same community are highly similar, while those from different communities exhibit weaker connections. Community detection plays a central role in numerous applications, including social network analysis, image segmentation, bioinformatics, recommendation systems, and document clustering.
\end{comment}

A widely adopted paradigm for community detection is spectral clustering. The key idea is to construct a low-dimensional embedding from the leading eigenvectors of the similarity matrix and subsequently apply a clustering algorithm, such as \(k\)-means, to identify the underlying communities. Owing to its simplicity, scalability, and strong theoretical guarantees, spectral clustering has become one of the most successful approaches for community detection. Nevertheless, its performance critically depends on the quality of the estimated spectral subspace. In practice, the observed similarity matrix is often contaminated by noise, outliers, or missing observations, which may substantially deteriorate the accuracy of the spectral embedding and consequently lead to unreliable clustering results. 

Spectral clustering computes the leading \(K\) eigenvectors \(U\in\mathbb{R}^{n\times K}\) of the affinity matrix \(A\) and performs clustering on the rows of \(U\). However, the embedding \(U\) is not unique: for any orthogonal matrix \(R\in\mathbb{R}^{K\times K}\), the transformed embedding \(UR\) preserves all pairwise Euclidean distances between rows and therefore yields exactly the same clustering result. Consequently, rather than seeking an optimal embedding \(U\), it is more natural to consider the projection matrix $X=UU^\top,$
which is invariant under orthogonal transformations of the embedding, i.e.,$(UR)(UR)^\top=UU^\top$. This representation eliminates the rotational ambiguity inherent in spectral embeddings and provides a unique characterization of the underlying \(K\)-dimensional subspace. An additional advantage of optimizing over the projection matrix \(X\) is that it naturally facilitates the incorporation of desirable structural priors through entrywise regularization. Such regularization effectively reduces the search space to projection matrices that better reflect the intrinsic structure of the clustering problem. To illustrate this point, consider an ideal affinity matrix defined by
\[
A_{ij}=
\begin{cases}
1, & c(i)=c(j),\\
0, & \text{otherwise},
\end{cases}
\]
where \(c(i)\) denotes the cluster label of the \(i\)-th sample and \(n_k\) is the size of the \(k\)-th cluster. The corresponding spectral projection matrix, up to a simultaneous permutation of its rows and columns, is given by
\[
X=
\begin{bmatrix}
\frac{1}{n_1}\mathbf{1}_{n_1\times n_1} & \mathbf{0} & \cdots & \mathbf{0}\\
\mathbf{0} & \frac{1}{n_2}\mathbf{1}_{n_2\times n_2} & \cdots & \mathbf{0}\\
\vdots & \vdots & \ddots & \vdots\\
\mathbf{0} & \mathbf{0} & \cdots & \frac{1}{n_K}\mathbf{1}_{n_K\times n_K}
\end{bmatrix},
\]
where \(\mathbf{1}_{n\times n}\) denotes the all-ones matrix. This ideal projection possesses several desirable structural properties: it is entrywise nonnegative, every entry is bounded by
$\frac{1}{\min_{1\le k\le K} n_k},$
and it exhibits a sparse block-diagonal structure with zero entries between different clusters. These observations suggest that high-quality projection matrices should approximately inherit these properties even in the presence of noise. Therefore, it is natural to incorporate suitable entrywise regularization into the projection matrix approximation framework to promote boundedness, nonnegativity, and sparsity, thereby improving the robustness and interpretability of the recovered spectral projection.

Motivated by the special structure of $X$, we propose a regularized projection
matrix approximation (RPMA) framework for recovering low-rank
projection matrices:

\begin{equation}
    \min_{X\in \mathcal P_K}\|X-A\|_{\rm F}^2+\lambda R(X).
\end{equation}
Here, 
$\mathcal P_K = \{X \in \mathbb{R}^{n \times n} : X = X^\top,\; X^2 = X,\; \operatorname{tr}(X) = K\}$ 
is the set of rank-$K$ orthogonal projection matrices, and 
$R(\cdot)$ is a regularization term that promotes desirable structural properties---such as sparsity, grouping, or smoothness---in the solution $X$. 
Since $\|X\|_{\rm F}^2 = K$ for all $X \in \mathcal P_K$, the RPMA objective is equivalent (up to a constant) to minimizing
\[
F_\lambda(X) = -2\langle X, A \rangle + \lambda R(X),\quad X\in\mathcal P_K.
\]

The proposed model augments the classical projection approximation formulation with a regularization term, thereby encouraging sparse and interpretable projection structures while simultaneously suppressing spurious noise. By choosing appropriate regularizers, such as the Huber penalty, the proposed approach is able to effectively mitigate the influence of outliers and recover cleaner projection matrices than those produced by conventional spectral projection methods.

%The optimization problem is formulated over the manifold of rank-\(K\) projection matrices, which is intrinsically related to the Grassmann manifold. Exploiting this geometric structure, we establish a comprehensive theoretical framework for the proposed model. In particular, we characterize the geometry of the projection matrix manifold, derive first- and second-order optimality conditions, and analyze the landscape of the regularized objective. Furthermore, we develop efficient manifold optimization algorithms, including a gradient projection method and a low-rank Cayley-transform-based Riemannian gradient method, both of which preserve feasibility throughout the iterations.

%Experiments on synthetic and real dataset demonstrate that the proposed RPMA framework consistently produces cleaner projection matrices and achieves superior clustering performance compared with conventional spectral methods, especially in the presence of sparse and heterogeneous noise.

The main contributions of this paper are summarized as follows:
\begin{itemize}
    \item We propose a novel regularized projection matrix approximation framework for robust community detection and clustering.
    
    \item We establish the geometric properties of the rank-K projection matrix manifold and derive the first- and second-order optimality conditions for the proposed model. Furthermore, we characterize the stability of the leading eigenspace and analyze the global landscape of all critical points.
    
    \item We develop Riemannian optimization algorithms based on gradient projection and further accelerate the optimization by employing Cayley transformations, which avoid the repeated eigendecompositions required by the gradient projection method.
    
    \item We provide experiments on synthetic and real dataset demonstrating that the proposed regularization significantly improves the recovery accuracy of projection matrices and enhances clustering performance.
\end{itemize}

\subsection{Related works}
%Low-rank matrix optimization with additional structural constraints is a common problem in machine learning and signal processing \cite{zhang2012low,zhang2012inducible}. The problem aims to find the best low-rank matrix approximation that also satisfies certain structural constraints, such as non-negativity, symmetry, boundedness, simplex and sparsity. There are two different appraoches to achieve this target. Firstly, these constraints can be achieved via matrix factorization with explicit constraints such as non-negative matrix factorization~\cite{lee1999learning}, semi-nonnegative matrix factorization~\cite{ding2008convex}, bounded low-rank matrix approximation~\cite{kannan2012bounded}, bounded projection matrix approximation~\cite{BPMA}.  Secondly, these constraints also can be achieved via the soft-regularization term routine such as simultaneously low-rank and sparse matrix approximation~\cite{zhang2022graph,ji2011robust,richard2012estimation,nie2016constrained}. These works, however, only seek a low-rank matrix as a low-dimensional embedding of the input similarity matrix. However, in this work, we no longer aim at learning the embedding but aim at recover a high quality group connection matrix, which is not necessarily a projection matrix through regularized projection matrix approximation.

The success of clustering heavily depends on the quality and cleanness of the affinity matrix. Numerous approaches have been proposed to refine or distill useful information from such matrices. Among them, low-rank approximation under additional structural constraints stands out as one of the most popular and effective strategies, with wide applications in machine learning and signal processing~\cite{zhang2012low,zhang2012inducible}. The objective is to find the best low-rank approximation of a given matrix that simultaneously satisfies certain structural properties, including non-negativity~\cite{pan2019generalized}, symmetry, boundedness~\cite{thanh2022bounded}, simplex constraints~\cite{abdolali2021simplex}, and sparsity~\cite{zhai2024sparse}. Two primary approaches have emerged to achieve this goal.
The first enforces structural constraints directly through matrix factorization with explicit constraints. Representative methods include non-negative matrix factorization~\cite{lee1999learning}, semi-nonnegative matrix factorization~\cite{ding2008convex}, bounded low-rank matrix approximation~\cite{kannan2012bounded}, and bounded projection matrix approximation~\cite{BPMA}. The second incorporates structural constraints via soft regularization terms added to the objective function, encompassing techniques such as simultaneous low-rank and sparse matrix approximation~\cite{zhang2022graph,ji2011robust,richard2012estimation,nie2016constrained}.
In contrast, our work pursues a distinct objective: rather than learning an embedding, we seek to recover a high-quality group connection matrix through regularized projection matrix approximation.

\section{Constraints and Penalties }

In this section, we detail several concrete instantiations of the penalty function $R(\cdot)$, each tailored to enforce a distinct structural prior on the solution. Depending on the functional form of $R(\cdot)$, the resulting regularization schemes naturally partition into two fundamental categories: \emph{entry-independent} penalties and \emph{entry-dependent} penalties. The former apply separable, element-wise operations to $X$---ideal for promoting local traits such as sparsity, non-negativity, or boundedness. The latter, however, couple multiple entries through global functions, enabling the imposition of holistic properties like row-sum constraints or energy normalization that transcend simple componentwise treatment.

Entrywise penalties are applied independently to each element of the matrix $X$ and are well-suited for enforcing local properties such as non-negativity, boundedness, or sparsity at the individual entry level. These constraints can be readily implemented by specifying $R(X):=\sum_{i,j}g(X_{i,j})$ as a separable function that acts element-wise on $X$.  

In contrast, certain structural requirements involve the collective behavior of multiple entries rather than isolated elements. For instance, constraints like normalizing the sum of each row to a prescribed constant cannot be captured by entrywise operations alone. Such scenarios call for a penalty function that depends on an aggregation of entries—e.g., a function of row-wise sums—thereby necessitating a more global formulation beyond the scope of separable penalties.

%use propose the penalty formsDepending on various motivations for $X$, such as non-negativity, boundedness, and sparsity, we have the flexibility to choose different formulations of $g(\cdot)$. In what follows, we give three different types of $g(\cdot)$:

\subsection{Entry-independent Penalties}
Some desired properties can be achieved by employing penalties that act independently on the entries of $X$. Such properties include boundedness, nonnegativity, and sparsity. We discuss each of these penalties below.

\paragraph{Bounded Penalty}
%From \eqref{assignP}, we know the low-rank approximation $X$ of an assignment matrix should be bounded in $[0,1/\min_k n_k]$. It is necessary for us to find the best bounded low-rank projection matrix. For general, we consider the problem of finding a projection matrix such that each of its elements in $[\alpha,\beta]$. The simplest idea is to introduce the indicator function $I(x)$ such that if $x\in [\alpha,\beta]$, $I(x)=0$, otherwise $I(x)=+\infty$. However, this function is non-continuous which causes difficulty to analyze with differentiable tools.
%Based on \eqref{assignP}, it's evident that the low-rank approximation $X$ of an assignment matrix must be confined within $[0,1/\min_k n_k]$. Hence, it's imperative for us to learn a bounded rank-$K$ projection matrix from $A$. In a general form, we address the problem as finding a projection matrix where each element falls within $[\alpha,\beta]$. One straightforward approach is to introduce the indicator function $I(x)$, defined such that $I(x)=0$ when $x$ lies in $[\alpha,\beta]$, and $I(x)=+\infty$ otherwise. However, this function's lack of continuity poses difficulties when employing differentiable tools for analysis.
For the clustering task, the rank-$K$ approximation of the underlying assignment matrix must be confined to the interval $[0, 1/\min_k n_k]$, where $n_k$ denotes the number of samples in the $k$-th cluster. This naturally motivates the introduction of a penalty function that enforces boundedness within $[\alpha, \beta]$. To this end, we adopt a smooth, convex function $g_{\alpha,\beta}(z)$ as a relaxation of the bounded-indicator function, defined as
\begin{equation}\label{bb}
g_{\alpha,\beta}(z) = (\min\{z-\alpha,0\})^2 + (\min\{\beta-z,0\})^2.
\end{equation}

The construction of $g_{\alpha,\beta}(z)$ reveals that it solely penalizes $z$ when it lies outside the range $[\alpha,\beta]$.
It can be observed that as $\lambda$ approaches $+\infty$, $\lambda g_{\alpha,\beta}(z)$ exhibits behavior akin to that of the indicator function. Moreover, this function is convex, and its derivative is Lipschitz continuous, satisfying $|g'_{\alpha,\beta}(x_1)-g'_{\alpha,\beta}(x_2)|\leq 2|x_1-x_2|$.

%From the construction of $g_{\alpha,\beta}(z)$, it is evident that $g_{\alpha,\beta}(z)$ only penalize $z$ when it falls out of the scope $[\alpha,\beta]$. It can be verified that $\lambda g_{\alpha,\beta}(z)$ behaves similarly to the indicator function as $\lambda \rightarrow +\infty$. Additionally, this function is convex, and its derivative is Lipschitz continuous with $|g'_{\alpha,\beta}(x_1)-g'_{\alpha,\beta}(x_2)|\leq 2|x_1-x_2|$.

\paragraph{Non-negativity Penalty}

The non-negativity constraint can be viewed as a one-side boundedness requirement by setting $\alpha=0,\beta=+\infty$. Thus, we can consider a non-negativity penalty function as a special case of the bounded penalty in \eqref{bb}. Therefore,
we introduce this non-negative penalty function to regularize the projection matrix, which is defined as
\begin{equation}\label{nn}
 g(z) :=  (\min\{z,0\})^2.
\end{equation}

%Similarly to the bounded penalty, the function $g(z)$ only exerts an effect when $z<0$. Furthermore, as $\lambda$ tends to infinity, $\lambda g(z)$ exhibits behavior akin to the indicator function $I(z)$, where $I(z)=0$ if and only if $z \geq 0$, and $I(z)=+\infty$ otherwise. Additionally, $g(z)$ is nonnegative and convex, and its derivative is Lipschitz continuous, satisfying $|g'(x_1)-g'(x_2)| = 2|x_1-x_2|$. %These conditions can ensure the convergence behavior of solving \eqref{Convex} via the ADMM method. 

%It is worth to mention that the non-negativity property can be seen as one-side bounded condition. In the following, we discuss the general bounded penalty function.
\paragraph{Sparsity Penalty}

The element-wise summation of absolute values $\sum_{i,j} |X_{i,j}|$ yields the Lasso penalty $\|X\|_1$~\cite{tibshirani1996regression}. However, the absolute value function is non-differentiable at $0$ and fails to satisfy the Lipschitz continuity property. To mitigate this issue, we approximate the absolute value function using the Huber loss~\cite{sun2020adaptive,huang2021robust}, which replaces the nonsmooth absolute value with a quadratic function in a neighborhood of the origin. The Huber loss is defined as
\begin{equation}\label{hl}
g_\delta(x)=\frac{1}{2\delta}\min(x^2,\delta^2)+\max(0,|x|-\delta)
\end{equation}
where $\delta > 0$ is a threshold parameter. It can be readily verified that the Huber function preserves the desirable properties of the absolute value penalty while being both differentiable and Lipschitz continuous everywhere, including at $x = 0$.

\subsection{Entry-dependent Penalties}
Some desirable structural properties cannot be enforced by entrywise penalties alone, as they depend on the interactions among multiple entries of the projection matrix. Such regularizers are referred to as \emph{entry-dependent penalties}. A representative example is the row-sum constraint commonly used in graph learning and clustering.

\paragraph{Row-Sum Penalties}
Enforcing the row sums of a similarity or projection matrix to equal one is a widely adopted normalization strategy in graph refinement and spectral clustering. This constraint can be naturally incorporated into the proposed RPMA framework through the regularization term
\[
R(X)=\|X\mathbf{1}-\mathbf{1}\|_2^2,
\]
where $\mathbf{1}$ denotes the all-ones vector. Minimizing this penalty encourages each row of $X$ to sum to one while allowing small deviations controlled by the regularization parameter.

When $X$ is further constrained to be a symmetric projection matrix, it is often desirable to enforce the column-sum constraint simultaneously. In this case, we employ the symmetric regularizer
\[
R(X)
=
\|X\mathbf{1}-\mathbf{1}\|_2^2
+
\|X^\top\mathbf{1}-\mathbf{1}\|_2^2,
\]
which encourages both the row and column sums of $X$ to be close to one. Since $X=X^\top$ on the rank-$K$ projection manifold, the two terms are identical. Nevertheless, the above formulation naturally extends to more general matrix optimization problems in which symmetry is not imposed.
%It is worth mentioning that as $\delta \rightarrow 0^+$, $g_\delta(x)$ uniformly converges to $|x|$. Thus, the Huber loss function can induce sparse solutions similarly to the Lasso penalty when $\delta$ is small enough. We propose using the summation of the Huber loss function $\sum_{i,j}g_\delta (X_{i,j})$ as a substitute for the Lasso penalty $\|X\|_1$ to learn a sparse projection approximation by adopting a sufficient small $\delta$. Compared to the absolute value function, the Huber loss function possesses specific advantages, such as smoothness and Lipschitz continuity, demonstrated by $|g'_\delta(x_1)-g'_\delta(x_2)| \leq \frac{1}{\delta}|x_1-x_2|.$

\section{Local Geometry and Stability Properties}
It is well known that the set of rank-\(K\) projection matrices is
diffeomorphic to the Grassmann manifold \(\mathrm{Gr}(K,n)\).
Moreover, \(\mathrm{Gr}(K,n)\) can be realized as the quotient of the
Stiefel manifold \(\mathrm{St}(K,n)\) by the natural right action of
the orthogonal group \(O(K)\), namely
$\mathrm{Gr}(K,n)
=
\mathrm{St}(K,n)/O(K)$~\cite{absil2008optimization}. The manifold of rank-K projection matrices provides a convenient representation for studying the geometric structure of the Grassmann manifold.

In this section, we first derive the first-order necessary condition for a critical point. We then establish the second-order condition, which distinguishes critical points as local minimizers or saddle points. Finally, we present a theorem characterizing the stability of the critical points and the global landscape of the optimization problem.

\subsection{First-Order Condition}

The projection matrix representation provides a convenient description
of the tangent space of \(\mathcal P_K\). Specifically, differentiating
the defining equation \(X^2=X\) yields the characterization of the
tangent space.

\begin{lemma}[Tangent space characterization] \label{tangent_space}
For any $X\in\mathcal P_K$,
\[
T_X\mathcal P_K
=
\{\Delta=\Delta^\top:\ X\Delta+\Delta X=\Delta\}.
\]
Equivalently, $T_X\mathcal P_K
=
\{[\Omega,X]:\Omega^\top=-\Omega\},
$
where the Lie bracket is defined as $[\Omega,X]:=\Omega X-X\Omega$.
\end{lemma}
\begin{proof}
    Let $X\in\mathcal P_K$ ($X^\top=X$, $X^2=X$) and let $X(t)\in\mathcal P_K$ with $X(0)=X$, $\dot X(0)=\Delta$. Differentiating $X(t)^2=X(t)$ and $X(t)^\top=X(t)$ gives
\[
X\Delta+\Delta X=\Delta,\qquad \Delta^\top=\Delta. \tag{1}
\]
Thus $T_X\mathcal P_K\subseteq\{\Delta=\Delta^\top:X\Delta+\Delta X=\Delta\}$.

Conversely, for any symmetric $\Delta$ satisfying (1), the curve
$X(t)=X+t\Delta$ remains in $\mathcal P_K$ to first order; a standard argument shows $\Delta\in T_X\mathcal P_K$. Hence the first characterization holds.

For the equivalence, let $\Delta=X\Omega-\Omega X$ with $\Omega^\top=-\Omega$. Then
\[
X\Delta+\Delta X=X(X\Omega-\Omega X)+(X\Omega-\Omega X)X=\Delta,
\]
using $X^2=X$, so $\Delta$ satisfies (1). Conversely, given $\Delta$ satisfying (1), choose $\Omega=[X,\Delta]=X\Delta-\Delta X$. Then $\Omega^\top=-\Omega$ and
\[
X\Omega-\Omega X=X(X\Delta-\Delta X)-(X\Delta-\Delta X)X=\Delta,
\]
since $X\Delta X=0$ from (1). Thus the two forms are equivalent. 
\end{proof}

 Based on the tangent-space characterization in Lemma~\ref{tangent_space}, Lemma~\ref{tangent} provides a closed-form representation of the orthogonal projection operator onto the tangent space.

\begin{lemma}[Orthogonal projection onto the tangent space]\label{tangent}
Let \(X\in\mathcal P_K\) and let \(H\in\mathbb S^n\) be any symmetric matrix. Then the orthogonal projection of \(H\) onto the tangent space \(T_X\mathcal P_K\), with respect to the Frobenius inner product, is given by
\[
\Pi_{T_X\mathcal P_K}(H)
=
XH(I-X)+(I-X)HX = [X,[X,H]].
\]
\end{lemma}
\begin{proof}
    Let $X\in\mathcal P_K$ ($X^\top=X$, $X^2=X$) and $X(t)\in\mathcal P_K$ with $X(0)=X$, $\dot X(0)=\Xi$. Differentiating $X^2=X$ and $X^\top=X$ gives
\[
X\Xi+\Xi X=\Xi,\qquad \Xi^\top=\Xi.
\]
Equivalently,
\[
\Xi=X\Xi(I-X)+(I-X)\Xi X. \tag{1}
\]
Conversely, for any skew-symmetric $\Omega$, $\Xi:=[X,\Omega]=X\Omega-\Omega X$ satisfies (1) since $X^2=X$, hence
\[
T_X\mathcal P_K=\{[X,\Omega]:\Omega^\top=-\Omega\}.
\]

For the projection, define $\Pi(H):=XH(I-X)+(I-X)HX\in T_X\mathcal P_K$. Its orthogonal complement is characterized by $N=XHX+(I-X)H(I-X)$, which commutes with $X$ and satisfies $\langle N,[X,\Omega]\rangle=0$ for all skew-symmetric $\Omega$. Finally,
\[
[X,[X,H]]=XH+HX-2XHX=\Pi(H).
\]
Thus $\Pi(H)=[X,[X,H]]$ is the orthogonal projection onto $T_X\mathcal P_K$. 
\end{proof}
Next, we derive the Riemannian first-order optimality condition. Recall that the Riemannian gradient is obtained by projecting the Euclidean gradient onto the tangent space of the manifold. Combining the tangent-space characterization in Lemma~\ref{tangent_space} with the tangent projection operator derived in Lemma~\ref{tangent}, we obtain a concise Lie-bracket representation of the Riemannian gradient, which leads directly to the first-order necessary optimality condition on the projection matrix manifold.
\begin{theorem}[Riemannian First-Order Optimality Condition] \label{Riemannian First}
$X\in\mathcal P_K$ is a critical point of $F$ if and only if $[X,\nabla F_\lambda(X)]=0$. Equivalently,
$\left[X,\,
A-\frac{\lambda}{2}G(X)
\right]
=0$, 
where $G(X)=\bigl(g'(X_{ij})\bigr)_{i,j}.$
\end{theorem}

\begin{proof}
    
For $F_\lambda(X)=-2\langle A,X\rangle+\lambda\sum_{i,j}g(X_{ij})$, the Euclidean gradient is
\[
\nabla F_\lambda(X)=-2A+\lambda G(X),\quad G(X)=(g'(X_{ij})).
\]
By Lemma~\ref{tangent}, $T_X\mathcal P_K=\{[X,\Omega]:\Omega^\top=-\Omega\}$. Criticality requires
\[
0=DF_\lambda(X)[\Delta]=\langle \nabla F_\lambda(X),\Delta\rangle_{\rm F}
\]
for all $\Delta\in T_X\mathcal P_K$. Taking $\Delta=[X,\Omega]$,
\[
\begin{aligned}
0=&\langle \nabla F_\lambda,[X,\Omega]\rangle_{\rm F}\\
=&{\rm tr} ((X\nabla F_\lambda-\nabla F_\lambda X)\Omega)
=\langle [X,\nabla F_\lambda],\Omega\rangle_{\rm F}
\end{aligned}
\]
for all skew-symmetric $\Omega$. Hence $[X,\nabla F_\lambda(X)]=0$. Substituting $\nabla F_\lambda$:
\[
[X,-2A+\lambda G(X)]=0
\;\Longleftrightarrow\;
[X,\,A-\tfrac{\lambda}{2}G(X)]=0.
\]

Conversely, if $[X,A-\tfrac{\lambda}{2}G(X)]=0$, then $[X,\nabla F_\lambda(X)]=0$, so $\langle \nabla F_\lambda,[X,\Omega]\rangle=0$ for all skew-symmetric $\Omega$, implying $DF_\lambda(X)[\Delta]=0$ for all $\Delta\in T_X\mathcal P_K$. Thus $X$ is critical. 
\end{proof}

\subsection{Second-Order Condition}
To derive the second-order necessary optimality condition for \(F_\lambda\), we require the Riemannian Hessian. Consider a smooth curve \(X(t)\subset\mathcal P_K\) with \(X(0)=X_\star\) and \(\dot X(0)=\Xi\), and define \(\varphi(t):=F_\lambda(X(t))\). Since \(X_\star\) is a local critical point, we have \(\varphi'(0)=\langle \nabla F_\lambda(X_\star),\Xi\rangle=0\) and \(\varphi''(0)\ge 0\). Differentiating yields
\[
\varphi'(t)=\langle \nabla F_\lambda(X(t)),\dot X(t)\rangle,
\]
and hence
\[
\varphi''(0)
=
\langle D(\nabla F_\lambda)(X_\star)[\Xi],\Xi\rangle
+
\langle \nabla F_\lambda(X_\star),\ddot X(0)\rangle.
\]
Because \(\nabla F_\lambda(X_\star)\) lies in the normal space \(N_{X_\star}\mathcal P_K\), the second term reduces to
\[
\langle \nabla F_\lambda(X_\star),\ddot X(0)\rangle
=
\langle \nabla F_\lambda(X_\star),\Pi_N(\ddot X(0))\rangle,
\]
where \(\Pi_N\) denotes the orthogonal projection onto the normal space. 

It is worth noting that $\Pi_N(\ddot X(0))$ is precisely the second fundamental form $II_X(\cdot,\cdot)$, which yields a closed-form expression for this normal projection in terms of the tangent velocity $\dot X(0)$:
\[
II_X: T_X\mathcal P_K \times T_X\mathcal P_K \to N_X\mathcal P_K,
\]
which encodes the extrinsic curvature of the manifold and plays a central role in the Riemannian Hessian and second-order optimality. Its explicit form is stated in Lemma~\ref{second_fundamental}.

\begin{lemma}[Second fundamental form of the Grassmann manifold] \label{second_fundamental}
For $\mathcal P_K$
with tangent and normal spaces at \(X\in\mathcal P_K\)
\[
\begin{aligned}
&T_X\mathcal P_K=\{\Xi\in\mathbb S^n: X\Xi+\Xi X=\Xi\},\qquad \\
&N_X\mathcal P_K=\{N\in\mathbb S^n: XN=NX\},
\end{aligned}
\]
the second fundamental form w.r.t. the Euclidean metric is
\[
II_X(\Xi,H)=-(2X-I)(\Xi H+H\Xi),\qquad \Xi,H\in T_X\mathcal P_K,
\]
and in particular \(II_X(\Xi,\Xi)=-2(2X-I)\Xi^2\).
\end{lemma}
\begin{proof}
Let $X(t)\subset\mathcal P_K$ be a smooth curve with $X(0)=X$ and $\dot X(0)=\Xi$. Since $X(t)^2=X(t)$, differentiating twice gives
\[
X\ddot X+\ddot X X-\ddot X=-2\Xi^2.
\]
Define $L_X(Z):=XZ+ZX-Z$. Then $L_X(\ddot X)=-2\Xi^2$. For $\Xi\in T_X\mathcal P_K$, $L_X(\Xi)=0$; for $N\in N_X\mathcal P_K$, $XN=NX$, so $L_X(N)=(2X-I)N$. Since $(2X-I)^2=I$, $L_X$ is invertible on $N_X\mathcal P_K$ with inverse $2X-I$. Decomposing $\ddot X=\Pi_T(\ddot X)+\Pi_N(\ddot X)$,
\[
\Pi_N(\ddot X)=L_X^{-1}(-2\Xi^2)=-2(2X-I)\Xi^2.
\]
By definition $II_X(\Xi,\Xi)=\Pi_N(\ddot X)$, so $II_X(\Xi,\Xi)=-2(2X-I)\Xi^2$. Polarization yields $II_X(\Xi,H)=-(2X-I)(\Xi H+H\Xi)$.
\end{proof}
Lemma~\ref{second_fundamental} enables the characterization of local extrema and saddle points of $F_\lambda$ on $\mathcal P_K$ via the associated quadratic form of the Riemannian Hessian in Lemma~\ref{lem:second_order_optimality_Flambda}.
\begin{lemma}[Second-Order Optimality Conditions for $F_\lambda$]
\label{lem:second_order_optimality_Flambda}
Let $A=A^\top\in\mathbb R^{n\times n}$ and $g\in C^2(\mathbb R)$.
Suppose $X_\star\in\mathcal P_K$ satisfies the first-order condition
$\operatorname{grad}F_\lambda(X_\star)=0$.
Then for every $\Xi\in T_{X_\star}\mathcal P_K$,
\[
\begin{aligned}
&\bigl\langle \operatorname{Hess}F_\lambda(X_\star)[\Xi],\Xi \bigr\rangle\\
=&
\bigl\langle D(\nabla F_\lambda)(X_\star)[\Xi],\Xi \bigr\rangle
-
2\bigl\langle (2X_\star-I)\nabla F_\lambda(X_\star),\Xi^2 \bigr\rangle.
\end{aligned}
\]
Consequently:
\begin{enumerate}
\item[(a)] \textbf{Local minimum:}
\[
\bigl\langle D(\nabla F_\lambda)(X_\star)[\Xi],\Xi \bigr\rangle
-
2\bigl\langle (2X_\star-I)\nabla F_\lambda(X_\star),\Xi^2 \bigr\rangle \ge 0.
\]
\item[(b)] \textbf{Local maximum:}
\[
\bigl\langle D(\nabla F_\lambda)(X_\star)[\Xi],\Xi \bigr\rangle
-
2\bigl\langle (2X_\star-I)\nabla F_\lambda(X_\star),\Xi^2 \bigr\rangle \le 0.
\]
\item[(c)] \textbf{Saddle point:} The quadratic form is indefinite.
\end{enumerate}
\end{lemma}

\begin{remark}
When $\lambda=0$, the Euclidean Hessian vanishes, i.e., $\nabla^2 F=0$, so the definiteness of the Riemannian Hessian is governed entirely by the second term in Lemma~\ref{lem:second_order_optimality_Flambda}. In particular, for $X_\star$ corresponding to the projection onto the subspace spanned by the top-$K$ eigenvectors of $A$, the Riemannian Hessian quadratic form reduces to
\[
\langle \operatorname{Hess}F_\lambda(X_\star)[\Xi],\Xi \rangle
=
-2\langle (2X_\star-I)\nabla F_\lambda(X_\star),\Xi^2 \rangle.
\]
A direct computation yields
\begin{equation}\label{hess}
\langle \operatorname{Hess}F_\lambda(X_\star)[\Xi],\Xi \rangle
=
2\sum_{i\in\mathcal I_K}\sum_{j\in\mathcal J_K}
(\lambda_i-\lambda_j)(\Xi_{ij})^2,
\end{equation}
where $\lambda_1\ge\cdots\ge\lambda_n$ are the eigenvalues of $A$, $\mathcal I_K=\{1,\ldots,K\}$, and $\mathcal J_K=\{K+1,\ldots,n\}$. Consequently, if the eigengap $\lambda_K-\lambda_{K+1}>0$, then the Hessian is negative definite, and $X_\star$ is a strict local maximum of $F_\lambda$ over $\mathcal P_K$.
\end{remark}

%Consequently, the second-order condition is determined jointly by the ambient-space Hessian and the intrinsic geometry of the manifold.

\begin{remark}
When $\lambda \neq 0$, the Riemannian Hessian incorporates an additional Euclidean term, $\lambda \sum_{i,j} g(X_{i,j})$, arising from the regularization term in the objective function. Under the assumption that the derivative of $g(\cdot)$ is Lipschitz continuous and its Hessian is uniformly bounded by a constant $\ell > 0$, the Riemannian Hessian retains its positive definiteness, provided that both $\lambda$ and $\ell$ satisfy suitable conditions (e.g., $\lambda$ is sufficiently large relative to the curvature of the manifold). This preservation of positive definiteness motivates a rigorous investigation into the stability of the minimizer, as it ensures local convexity and robustness of the solution under small perturbations.
\end{remark}
\subsection{Stability and Global Landscape}
From \eqref{hess}, for $\lambda=0$, there is exactly one local minimum and one local maximum, with all other critical points being saddles. Consequently, both the local minimum and maximum are global whenever the spectral gap is present. It is natural to expect that the global landscape of critical points remains stable under small perturbations of the Riemannian Hessian. In what follows, we present a theorem that quantifies the stability of the global minimizer.

\begin{theorem}[Local stability of the leading eigenspace]
\label{thm:stability}

Let $A=A^\top\in\mathbb R^{n\times n}$ have eigenvalues $\lambda_1\ge\cdots\ge\lambda_n$ with $K$-th eigengap $\eta_K:=\lambda_K-\lambda_{K+1}>0$. Let $R(X)=\sum_{i,j}g(X_{ij})$, where $g\in C^2(\mathbb R)$ satisfies $0\le g''\le\ell$ and $|g'|\le M$. Denote by $X_0=U_KU_K^\top$ the projector onto the leading $K$-dimensional eigenspace of $A$. Then there exists $c>0$, depending only on $A,g,n$, such that if $\lambda(\ell+2Mn)<2\eta_K$, the following hold:
\begin{enumerate}
\item There exists a unique critical point $X_\star(\lambda)\in\mathcal P_K$ near $X_0$.
\item The map $\lambda\mapsto X_\star(\lambda)$ is $C^1$ with $X_\star(0)=X_0$.
\item $\|X_\star(\lambda)-X_0\|_{\mathrm F}\le C\,\lambda/\eta_K$ for some $C>0$ independent of $\lambda$.
\item The Riemannian Hessian of $F_\lambda$ at $X_\star(\lambda)$ is positive definite; hence $X_\star(\lambda)$ is a strict local minimizer on $\mathcal P_K$.
\end{enumerate}
\end{theorem}

\begin{proof}
For $\lambda=0$, $F_0(X)=-2\langle A,X\rangle$. By Ky Fan's inequality, $\langle A,X\rangle\le\sum_{i=1}^K\lambda_i$, with equality iff $X=X_0$; hence $X_0$ is the unique minimizer of $F_0$.

The Euclidean gradient is $\nabla F_\lambda(X)=-2A+\lambda\nabla R(X)$, where $(\nabla R(X))_{ij}=g'(X_{ij})$. By Theorem~\ref{Riemannian First}, $X$ is critical iff $[X,\nabla F_\lambda(X)]=0$. Define
\[
\Phi(X,\lambda):=[X,A-\tfrac{\lambda}{2}\nabla R(X)].
\]
Then $\Phi(X_0,0)=0$. By orthogonal invariance, take $A=\Lambda=\operatorname{diag}(\lambda_1,\ldots,\lambda_n)$. For a tangent vector $\Xi=\begin{pmatrix}0&B^\top\\B&0\end{pmatrix}$ in the block decomposition $\Lambda=\operatorname{diag}(\Lambda_1,\Lambda_2)$, one computes
\[
[\Xi,\Lambda]=
\begin{pmatrix}
0&(\Lambda_2-\Lambda_1)B^\top\\
(\Lambda_1-\Lambda_2)B&0
\end{pmatrix},
\]
so $\|[\Xi,\Lambda]\|_{\mathrm F}\ge \eta_K\|\Xi\|_{\mathrm F}$. Since $D_X\Phi(X_0,0)[\Xi]=[\Xi,A]$, the derivative is injective and hence invertible. The implicit function theorem yields a unique smooth branch $X_\star(\lambda)$ near $\lambda=0$ with $X_\star(0)=X_0$. Expanding $\Phi(X_\star(\lambda),\lambda)=0$ gives
\[
D_X\Phi(X_0,0)[X_\star(\lambda)-X_0]
=
\frac{\lambda}{2}[X_0,\nabla R(X_0)]+O(\lambda^2),
\]
and since $\|D_X\Phi(X_0,0)^{-1}\|_{\rm op}\le \eta_K^{-1}$, we obtain $\|X_\star(\lambda)-X_0\|_{\mathrm F}\le C\lambda/\eta_K$.

For the Hessian, recall
\[
\begin{aligned}
&\langle \operatorname{Hess}F_\lambda(X_0)[\Xi],\Xi\rangle\\
=&
\langle \nabla^2F_\lambda(X_0)[\Xi],\Xi\rangle
-
2\langle (2X_0-I)\nabla F_\lambda(X_0),\Xi^2\rangle.
\end{aligned}
\]
At $\lambda=0$, $\nabla^2F_0=0$, $\nabla F_0=-2A$, hence
\[
\langle \operatorname{Hess}F_0(X_0)[\Xi],\Xi\rangle
=
4\sum_{i\le K}\sum_{j>K}(\lambda_i-\lambda_j)B_{ji}^2
\ge 2\eta_K\|\Xi\|_{\rm F}^2.
\]
The difference satisfies
\[
\begin{aligned}
&\left|\langle \operatorname{Hess}F_\lambda(X_0)[\Xi],\Xi\rangle-\langle \operatorname{Hess}F_0(X_0)[\Xi],\Xi\rangle\right|\\
\le &\lambda(\ell+2Mn)\|\Xi\|_{\rm F}^2,
\end{aligned}
\]
using $0\le g''\le\ell$, $\|\nabla R(X_0)\|_{\rm F}\le Mn$, and $\|\Xi^2\|_*=\|\Xi\|_{\rm F}^2$. Thus
\[
\langle \operatorname{Hess}F_\lambda(X_0)[\Xi],\Xi\rangle
\ge \bigl(2\eta_K-\lambda(\ell+2Mn)\bigr)\|\Xi\|_{\rm F}^2.
\]
If $\lambda(\ell+2Mn)<2\eta_K$, the Hessian is positive definite at $X_0$. By smoothness, it remains positive definite near $X_0$, so $X_\star(\lambda)$ is a strict local minimizer.
\end{proof}

Similarly, replacing the eigengap $\eta_K$ associated with the leading
eigenspace by the minimum spectral gap
\[
\gamma:=\min_{i\neq j}|\lambda_i-\lambda_j|
\]
allows the perturbation argument to be applied to every PCA projector.
Consequently, the entire critical-point landscape of the unregularized
problem is preserved under sufficiently small regularization.

\begin{corollary}[Stability of the critical-point landscape]
\label{cor:global_landscape_stability}

Let $A=A^\top\in\mathbb R^{n\times n}$ have simple eigenvalues $\lambda_1>\cdots>\lambda_n$, and set $\gamma:=\min_{i\neq j}|\lambda_i-\lambda_j|>0$. Let $R(X)=\sum_{i,j}g(X_{ij})$, where $g\in C^2(\mathbb R)$, $0\le g''\le\ell$, and $|g'|\le M$. Then there exists $c>0$, depending only on $A,g,n$, such that if $\lambda(\ell+2Mn)<c\gamma$, the following hold:
\begin{enumerate}
\item Each PCA projector of the unregularized problem gives rise to a unique nearby critical point of $F_\lambda$, depending smoothly on $\lambda$.
\item The critical points of $F_\lambda$ are in one-to-one correspondence with the PCA projectors; hence $\#\operatorname{Crit}(F_\lambda)=\binom{n}{K}$.
\item The Morse index of every critical point is preserved under perturbation. In particular, the perturbation of the leading eigenspace remains a strict local minimizer, while perturbations of all other PCA projectors remain strict saddle points.
\end{enumerate}
\end{corollary}

\begin{algorithm}[t!]
\caption{Riemannian Gradient Projection Algorithm with Backtracking}
\label{alg:GPA}
\begin{algorithmic}[1]

\Require $X_0\in\mathcal P_K$, $\eta_{\max}>0$, $0<\beta,\sigma<1$

\For{$t=0,1,2,\ldots$}
    \State Compute Riemannian gradient $G_t=[X,[X_t, \nabla F_\lambda(X_t)]]$ and set $\eta_t=\eta_{\max}$.
    \While{true}
        \State $Y_t=X_t-\eta_t G_t$.
        \State Compute $X_{\rm new}= R_{\mathcal P_K}(Y_t)$.
        \If{$F_\lambda(X_{\rm new})
        \le F_\lambda(X_t)-\sigma\eta_t\|G_t\|_{\rm F}^2$}
            \State $X_{t+1}=X_{\rm new}$; \textbf{break}.
        \Else
            \State $\eta_t\leftarrow\beta\eta_t$.
        \EndIf
    \EndWhile
\EndFor
\end{algorithmic}
\end{algorithm}

\section{Algorithm}
In this section, we investigate numerical methods for solving optimization problems on the projection matrix manifold $\mathcal P_K$. In particular, we focus on two widely used approaches: the gradient projection method and the Cayley-transformation-based method.

The gradient projection method first performs a Euclidean gradient step and then projects the updated iterate back onto $\mathcal P_K$. This projection requires the computation of a rank-$K$ eigendecomposition, resulting in a per-iteration complexity of $O(n^2K)$.

Motivated by the Cayley transformation for optimizing on the Stiefel manifold~\cite{WenYin2013}, we also consider a Cayley-transformation-based update on the Grassmann manifold. By exploiting the geometric structure of the manifold, the Cayley update preserves feasibility without requiring an eigendecomposition. The dominant computational cost arises from forming matrix products of size $n\times K$ and computing the inverse of a $K\times K$ matrix, yielding an overall complexity of $O(nK^2)+O(K^3)$ per iteration. Consequently, when $n\gg K$, the Cayley-based method is significantly more efficient and scalable than the gradient projection approach.

\subsection{Gradient Projection Algorithm}
%\subsection{Gradient and Projection Operators}

The Euclidean gradient of $F_\lambda(\cdot)$ is
$\nabla F_\lambda(X)
=
-2A
+
\lambda G(X)$,
where $G(X) = \bigl(g'(X_{ij})\bigr)_{i,j}$. Given the current iterate $X_t$, a Riemannian gradient descent step with stepsize $\eta_t>0$ produces
\[
Y_t
=
X_t-\eta_t[X, [X,\nabla F_\lambda(X_t)].
\]
Since $Y_t$ generally does not belong to $\mathcal P_K$, we pull it back onto the feasible set via the optimization:
\begin{equation}\label{update}
X_{t+1}
=
R_{\mathcal P_K}(Y_t)=
\arg\min_{X\in\mathcal P_K}
\|X-Y\|_{\rm F}^2.
\end{equation}
The stepsize \(\eta_t\) is chosen via backtracking line search with parameters
\(0<\beta,\sigma<1\). Starting from \(\eta_{\max}>0\), we iteratively set
\(\eta_t\leftarrow \beta\eta_t\) until the Armijo condition
\begin{equation}\label{eq:armijo}
F_\lambda(X_{t+1})
\le
F_\lambda(X_t)
-
\sigma \eta_t \|\nabla F_\lambda(X_t)\|_{\rm F}^2
\end{equation}
holds.

Each iteration is dominated by the projection onto \(\mathcal P_K\), which requires computing the leading \(K\) eigenvectors of an \(n\times n\) matrix. A full eigendecomposition costs \(O(n^3)\), while Krylov/Lanczos methods reduce this to \(O(n^2K)\) for \(K\ll n\). The gradient evaluation costs \(O(n^2)\), so the projection step is the bottleneck.
\begin{comment}

The update \eqref{eq:GPA} preserves feasibility, i.e.,
\[
X_t\in\mathcal P_K,\quad \forall t\ge0.
\]
Since \(\mathcal P_K\) is compact and \(F\in C^1\), every accumulation point of \(\{X_t\}\) is stationary and satisfies
\[
[X,\nabla F_\lambda(X)]=0.
\]

Under the landscape stability condition \(\lambda\ell<c\gamma\), the objective admits a unique local minimizer, while all other critical points are strict saddles. Hence standard results for projected-gradient methods on smooth manifolds imply convergence to the global minimizer from almost every initialization.

Moreover, near \(X_\star(\lambda)\), the Riemannian Hessian is positive definite, yielding local linear convergence: for sufficiently small fixed \(\eta>0\),
\[
\|X_{t+1}-X_\star(\lambda)\|_{\rm F}
\le
\rho\,\|X_t-X_\star(\lambda)\|_{\rm F},
\qquad \rho\in(0,1).
\]
\end{comment}

\subsection{Cayley Transformation-based Riemannian Gradient Method}

Besides the classical gradient projection method, we consider a more efficient approach that constructs a smooth curve directly on the projection manifold. This avoids repeated projections onto the manifold via eigenvalue decompositions at each iteration, thereby significantly reducing computational overhead.

\begin{algorithm}[t!]
\caption{Cayley--SMW Riemannian Gradient Method}
\label{alg:cayley}
\begin{algorithmic}[1]

\Require $U_0^\top U_0=I_K$, $\tau_{\max}>0$, $0<\beta,\sigma<1$

\For{$k=0,1,\ldots$}

\State $X_k=U_kU_k^\top$

\State
\[
\Xi_k
=
(I-U_kU_k^\top)
\nabla F_\lambda(X_k)U_k
\]

\If{$\|\Xi_k\|_{\rm F}<{\rm tol}$}
\State \textbf{stop}
\EndIf

\State
\[
Y_k=[\,\Xi_k,\;U_k\,],\qquad
J=
\begin{bmatrix}
0&I_K\\
-I_K&0
\end{bmatrix}
\]

\State $\tau_k=\tau_{\max}$

\While{$\tau_k>\tau_{\min}$}

\State
\[
M_k
=
I_{2K}
+\frac{\tau_k}{2}JY_k^\top Y_k
\]

\State Solve
\[
M_kZ_k=JY_k^\top U_k
\]

\State
\[
\widetilde U
=
U_k-\tau_kY_kZ_k
\]

\If{
$
F_\lambda(\widetilde U\widetilde U^\top)
\le
F_\lambda(X_k)
-\sigma\tau_k\|\Xi_k\|_{\rm F}^2
$
}
\State $U_{k+1}=\widetilde U$
\State \textbf{break}
\EndIf

\State $\tau_k\leftarrow\beta\tau_k$

\EndWhile

\EndFor

\State Return $X^\star=U^\star(U^\star)^\top$

\end{algorithmic}
\end{algorithm}

Recall that the Grassmann manifold $\mathrm{Gr}(K,n)$ can be viewed as the quotient space of the Stiefel manifold under the right action of the orthogonal group $O(K)$. The Riemannian gradient of $F_\lambda(X)$ on $\mathrm{Gr}(K,n)$, evaluated at $X = U_k U_k^\top$, admits the representation
\[
\Xi_k = (I - U_k U_k^\top)\,\nabla F_\lambda(U_k U_k^\top)\,U_k,
\]
where $\Xi_k \in T_{U_k}\mathrm{Gr}(K,n)$ denotes the Riemannian gradient in the tangent space at $U_k$.

To perform Riemannian optimization on $\mathrm{Gr}(K,n)$, we construct a Cayley-type retraction based on an orthogonal transformation. Specifically, define the skew-symmetric matrix
\[
\Omega_k = U_k \Xi_k^\top - \Xi_k U_k^\top,
\qquad \Omega_k^\top = -\Omega_k.
\]
Then the Cayley transform generates an orthogonal matrix
\[
Q_k
=
\left(I - \frac{\tau_k}{2}\Omega_k\right)^{-1}
\left(I + \frac{\tau_k}{2}\Omega_k\right),
\]
which preserves orthogonality up to higher-order terms in $\tau_k$. The updated iterate is given by
\[
U_{k+1} = Q_k U_k,
\]
which ensures $U_{k+1} \in \mathrm{St}(K,n)$ and consequently $U_{k+1}U_{k+1}^\top \in \mathrm{Gr}(K,n)$.

\paragraph{Directional derivative}
Here, we derive the directional derivative of the Cayley curve at $\tau = 0$, and quantify the rate of variation of $X(\tau)$ at $\tau = 0$ in terms of the order $\mathcal{O}(\|\Xi\|_{\rm F}^2)$.\begin{proposition}[Descent Property of the Cayley Curve] \label{Cayley descent}
Let
\[
X=UU^\top\in\mathcal P_K,
\qquad
\Xi=(I-UU^\top)\nabla F_\lambda(X)U,
\]
and define $\Omega=-\Xi U^\top+U\Xi^\top$ . Consider the Cayley curve
\[
U(\tau)
=
\left(I-\frac{\tau}{2}\Omega\right)^{-1}
\left(I+\frac{\tau}{2}\Omega\right)U,
\qquad
X(\tau)=U(\tau)U(\tau)^\top .
\]
Then \(X(\tau)\in\mathcal P_K\) for all sufficiently small \(\tau\), and
\[
\frac{d}{d\tau}F_\lambda(X(\tau))
\Big|_{\tau=0}
=
-2\|\Xi\|_{\rm F}^2.
\]
In particular, if \(\Xi\neq 0\), then \(X(\tau)\) is a strict descent curve for \(F_\lambda\).
\end{proposition}
\begin{proof}
Since $\Omega^\top=-\Omega$, the Cayley transform  
$Q(\tau)=(I-\frac{\tau}{2}\Omega)^{-1}(I+\frac{\tau}{2}\Omega)$ is orthogonal. Thus $U(\tau)=Q(\tau)U$ satisfies $U(\tau)^\top U(\tau)=I_K$, so $X(\tau)=U(\tau)U(\tau)^\top\in\mathcal P_K$.

Differentiating at $\tau=0$: using $(A^{-1})'=-A^{-1}A'A^{-1}$ with $A=I-\frac{\tau}{2}\Omega$, we get $Q'(0)=\Omega$, hence $U'(0)=\Omega U$. Since $\Omega=-\Xi U^\top+U\Xi^\top$ and $U^\top\Xi=0$, we have $\Omega U=-\Xi$, so $U'(0)=-\Xi$. Differentiating $X(\tau)$ gives 
\[
X'(0)=-\Xi U^\top-U\Xi^\top.
\] 
By the chain rule and symmetry of $\nabla F_\lambda(X)$,
\[
\left.\frac{d}{d\tau}F_\lambda(X(\tau))\right|_{\tau=0}
=-2\langle \nabla F_\lambda(X)U,\Xi\rangle.
\]
Using $\Xi=(I-UU^\top)\nabla F_\lambda(X)U$ and $UU^\top\Xi=0$,
\[
\langle \nabla F_\lambda(X)U,\Xi\rangle
=\langle (I-UU^\top)\nabla F_\lambda(X)U,\Xi\rangle
=\|\Xi\|_{\rm F}^2.
\]
Thus the derivative equals $-2\|\Xi\|_{\rm F}^2$. If $\Xi\neq0$, this is negative, so the Cayley curve along the negative Riemannian gradient is a strict descent direction.
\end{proof}

Based on this proposition, a suitable step size $\tau$ can be selected via a backtracking line search, which balances sufficient descent with ease of optimization. Specifically, we initialize $\tau_k$ with $\tau_{\max} > 0$ and iteratively update
\[
\tau_k \leftarrow \beta \tau_k, \qquad 0 < \beta < 1,
\]
until the Armijo condition
\[
F_\lambda(U_{k+1}U_{k+1}^\top)
\le
F_\lambda(U_kU_k^\top)
-
\sigma \tau_k \|\Xi_k\|_{\rm F}^2
\]
is satisfied for some $0 < \sigma < 1$.
\paragraph{A faster approach for the Cayley transformation}
Noticing that the explicit representation of $Q_k$ involve the computation of the inverse $n\times n$ matrix $\left(I - \frac{\tau_k}{2}\Omega_k\right)$, which cost $O(n^3)$.
To avoid forming $Q_k \in \mathbb R^{n\times n}$ explicitly, we exploit the low-rank structure of $\Omega_k$. Define
\[
Y_k = [\,\Xi_k,\; U_k\,]\in\mathbb R^{n\times 2K}, \qquad
J =
\begin{bmatrix}
0 & I_K \\
-I_K & 0
\end{bmatrix},
\]
so that
\[
\Omega_k =- Y_k J Y_k^\top.
\]

Using the Sherman--Morrison--Woodbury formula, the Cayley update can be written in low-rank form as
\[
U_{k+1}
=
U_k
-
\tau_k Y_k
\left(
I + \frac{\tau_k}{2} J Y_k^\top Y_k
\right)^{-1}
J Y_k^\top U_k.
\]
\subsection{Complexity and Convergence Analysis}
\paragraph{Complexity}This implementation only requires the inversion of a $2K\times 2K$ matrix, leading to a per-iteration complexity of
\[
O(nK^2) + O(K^3).
\]

Compared with eigendecomposition-based methods, the computational complexity is reduced from $O(n^2K)$ to $O(nK^2) + O(K^3)$. In typical applications, the embedding dimension (or number of clusters) satisfies $K \ll n$, making the proposed approach significantly more efficient in large-scale settings.

\paragraph{Convergence.}
Assume that $F_\lambda\in C^1$ has a Lipschitz continuous gradient on the compact manifold $\mathcal P_K$. Then every accumulation point $X_\star$ of $\{X_k\}$ is a stationary point satisfying
\[
[X_\star,\nabla F_\lambda(X_\star)]=0,
\qquad
\|\Xi_k\|_{\rm F}\to 0.
\]
Furthermore, if $X_\star$ is a nondegenerate local minimizer with
\[
\mathrm{Hess}_{\mathcal P_K}F_\lambda(X_\star)\succeq \mu I,
\qquad \mu>0,
\]
then the iterates converge locally linearly to $X_\star$, i.e.,
\[
\|X_{k+1}-X_\star\|_{\rm F}
\le
\rho\,\|X_k-X_\star\|_{\rm F},
\qquad
\rho\in(0,1),
\]
for all sufficiently large $k$.

\section{Numerical Experiment}

\begin{figure*}[t]
%{\hspace{-0.09\linewidth}
\centering
\includegraphics[width=0.9\linewidth]{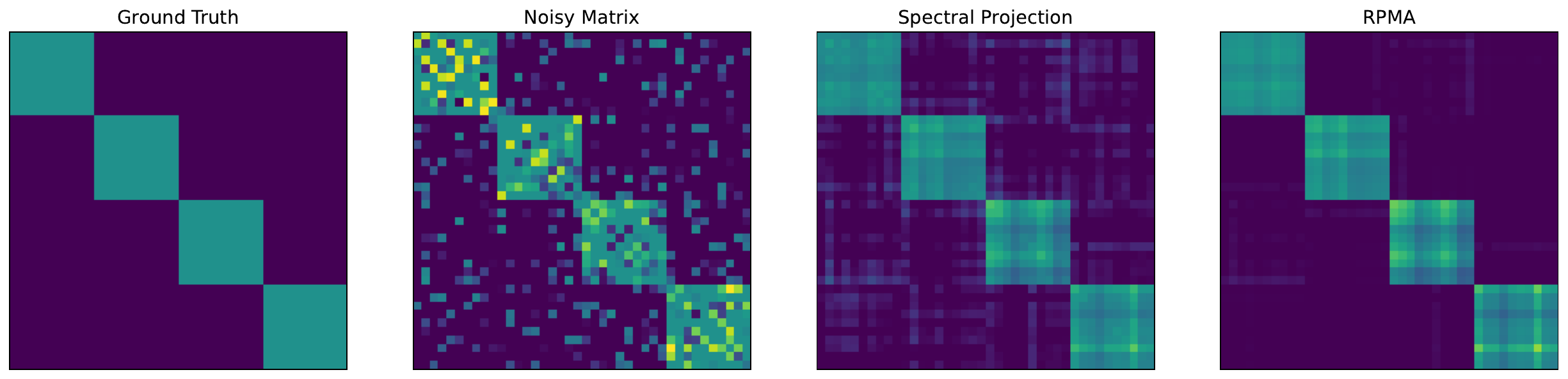}
\caption{Comparison of the spectral projection estimator and RPMA-S for projection matrix recovery. RPMA-S exploits sparsity regularization to recover a cleaner and sparser projection matrix. \label{solutions_compare}}
\end{figure*}

In this section, we first perform synthetic experiments to demonstrate
the effectiveness of the proposed regularized projection matrix approximation
(RPMA) framework. Specifically, we show that RPMA can recover the underlying
clean projection matrix more accurately than the classical spectral projection
estimator in the presence of perturbations. We consider two variants of the
RPMA model. 

The first variant, referred to as RPMA-S, incorporates an entry-wise robust
regularization term and is formulated as
\[
\min_{X\in {\mathcal P}_K}
-2 \langle X,A \rangle
+\lambda \sum_{i,j}g_\delta(X_{i,j}),
\]
where \(g_\delta(\cdot)\) denotes the Huber loss penalty with parameter
\(\delta\). This formulation aims to exploit the inherent sparsity or
regularity of the projection matrix while preserving its rank-\(K\) projection
structure.

To further incorporate additional structural information, we introduce the
structure-enhanced variant RPMA-NS, which is defined as
\[
\begin{aligned}
\min_{X\in {\mathcal P}_K}\quad
&-2 \langle X,A \rangle
+\lambda \sum_{i,j}g_\delta(X_{i,j})  \\
&+\mu \sum_{i,j}(\min\{X_{i,j},0\})^2
+\rho\|X{\bf 1}_n-{\bf 1}_n\|_2^2 .
\end{aligned}
\]
The additional two regularization terms enforce non-negativity and the
row-sum-to-one constraint, respectively, thereby incorporating prior
structural information of the projection matrix. The coefficients
\(\lambda\), \(\mu\), and \(\rho\) control the trade-off between data fidelity
and structural regularization.

We then evaluate the clustering performance of RPMA-S and
its structure-enhanced variant, RPMA-NS, on four real-world datasets:
COIL20, the AT\&T Database of Faces, Semeion, and DIGIT-10. The
experiments are organized into balanced-sample and imbalanced-sample
settings to assess both the clustering effectiveness and the robustness
of the proposed methods under homogeneous and heterogeneous class-size
distributions.\footnote{The source code for RPMA-S and RPMA-NS,
together with the experimental scripts, is publicly available at
\url{https://github.com/liangzhuan/RPMA}.}

%\footnote{\url{https://archive.ics.uci.edu/ml/datasets/Multiple+Features.https://www.kaggle.com/datasets/jessicali9530/coil100.https://archive.ics.uci.edu/dataset/240/human+activity+recognition+using+smartphones}.}
\subsection{Experiment on synthetic data}
%Firstly, we compare our algorithm with competitors on a subset of each dataset, which we denote as DIGIT-5, COIL-10, and HAR-3. These subsets consist of five, ten, and three classes, respectively.
%Subsequently, we utilize the entire data from each dataset, denoting them as DIGIT-10, COIL-20, and HAR-6, comprising ten, twenty, and six classes, respectively.

\begin{figure*}[!t]
    \centering
    \subfloat[Reconstruction error of RPMA-S, measured by 
    $\|\widehat{X}-X^*\|_{\rm F}$, where $\widehat{X}$ denotes the solution of the Huber-loss-regularized projection approximation model under different settings of the regularization parameter $\lambda$ and the Huber loss parameter $\delta$.
    \label{improve1}]
    {
        \includegraphics[width=0.46\linewidth]{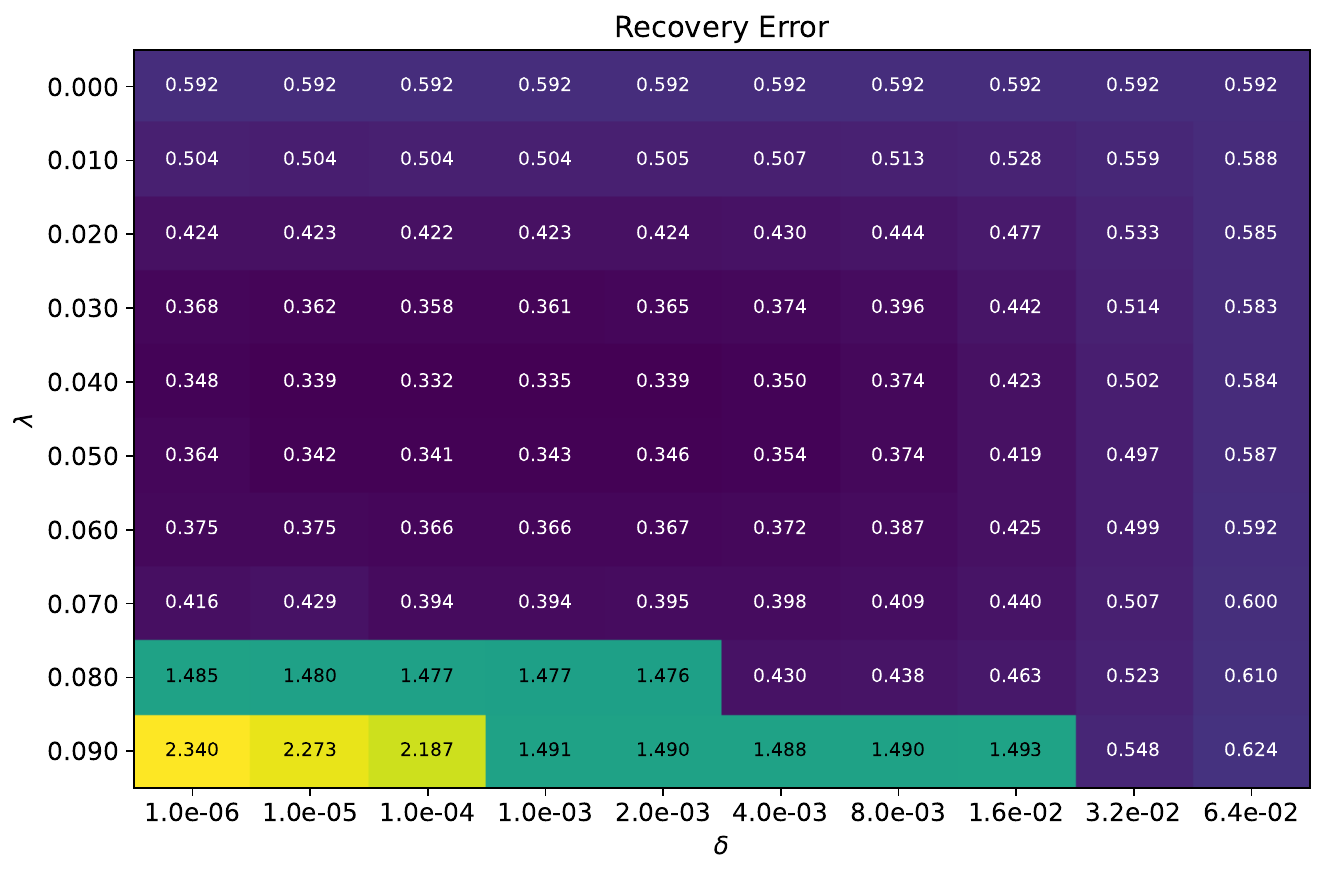}
        \label{fig:error_heatmap}
    }
    \hfill
    \subfloat[Relative recovery improvement of RPMA-S over the spectral projection estimator, measured by
    $\frac{\|X_{\rm spe}-X^*\|_{\rm F}-\|\widehat{X}-X^*\|_{\rm F}}
    {\|X_{\rm spe}-X^*\|_{\rm F}}$,
    where $\widehat{X}$ denotes the solution of the Huber-loss-regularized projection approximation model.
    \label{improve2}]
    {
        \includegraphics[width=0.46\linewidth]{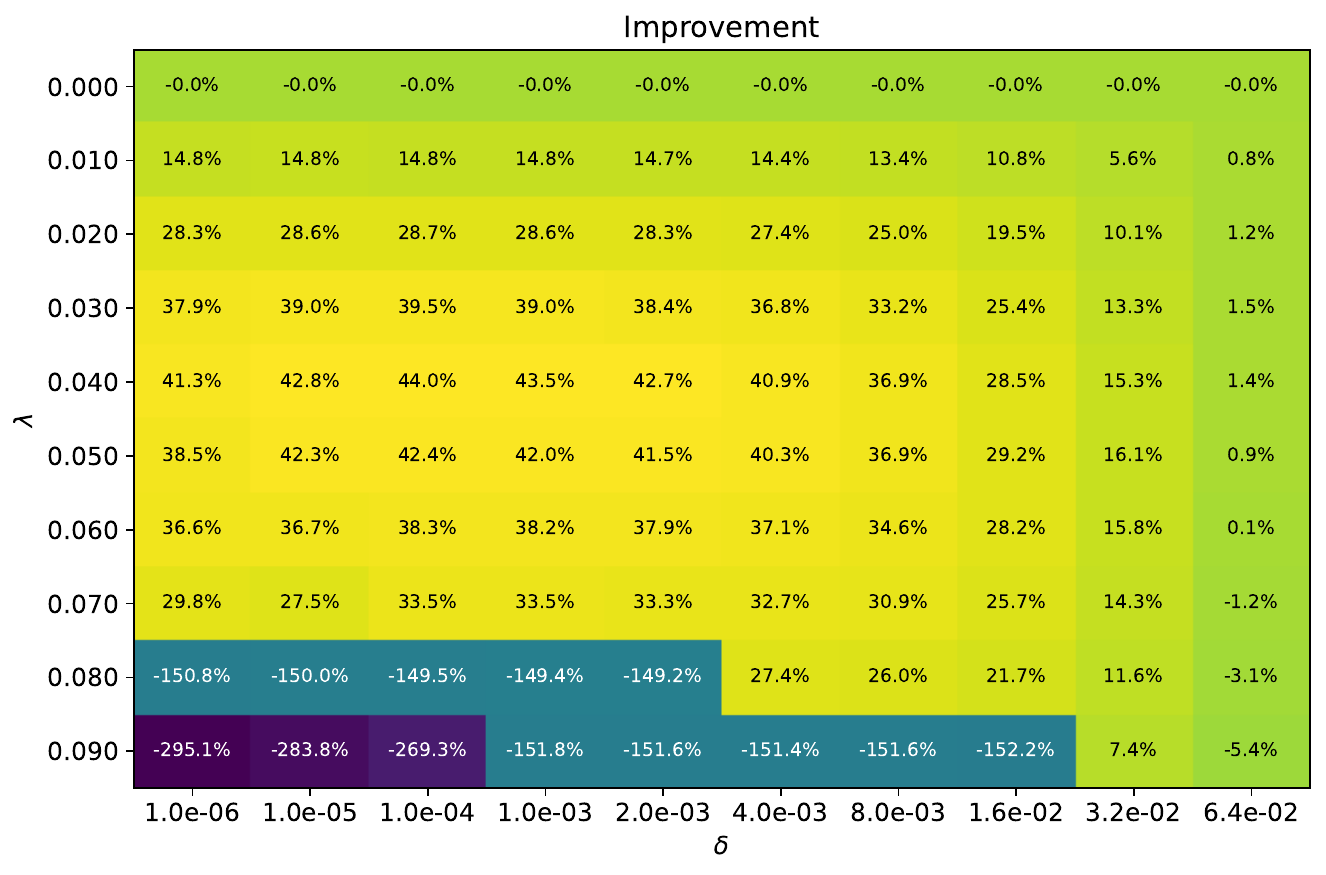}
        \label{fig:improvement_heatmap}
    }
    \caption{Performance comparison under different settings, initialized with the top-$K$ eigenvectors of $X$.}
    \label{fig:heatmaps}
\end{figure*}

\begin{table*}[h!]
\centering
\caption{Recovery error (left) and relative improvement over spectral projection (right) for different choices of $\lambda$ and $\delta$}
\label{tab:lambda_delta_canonical_basis}
\renewcommand{\arraystretch}{1.15}
\setlength{\tabcolsep}{3.8pt}
\resizebox{0.95\linewidth}{!}{
\begin{tabular}{c|ccccccccc|ccccccccc}
\toprule
&
\multicolumn{9}{c|}{Recovery Error $\|\widehat X-X^*\|_{\rm F}$}
&
\multicolumn{9}{c}{Improvement (\%)}\\
\cmidrule(lr){2-10}\cmidrule(lr){11-19}

\diagbox{$\lambda$}{$\delta$}
& $10^{-5}$ & $10^{-4}$ & $10^{-3}$ & $0.002$ & $0.004$
& $0.008$ & $0.016$ & $0.032$ & $0.064$
& $10^{-5}$ & $10^{-4}$ & $10^{-3}$ & $0.002$ & $0.004$
& $0.008$ & $0.016$ & $0.032$ & $0.064$\\
\midrule

0.00 &
0.592 & 0.592 & 0.592 & 0.592 & 0.592 & 0.592 & 0.592 & 0.592 & 0.592 &
0.0 & 0.0 & 0.0 & 0.0 & 0.0 & 0.0 & 0.0 & 0.0 & 0.0 \\

0.01 &
0.504 & 0.504 & 0.504 & 0.505 & 0.507 & 0.513 & 0.528 & 0.559 & 0.588 &
15.0 & 15.0 & 15.0 & 15.0 & 14.0 & 13.0 & 11.0 & 6.0 & 1.0 \\

0.02 &
0.423 & 0.422 & 0.423 & 0.424 & 0.430 & 0.444 & 0.477 & 0.533 & 0.585 &
29.0 & 29.0 & 29.0 & 28.0 & 27.0 & 25.0 & 20.0 & 10.0 & 1.0 \\

0.03 &
0.362 & 0.358 & 0.361 & 0.365 & 0.374 & 0.396 & 0.442 & 0.514 & 0.583 &
39.0 & 39.0 & 39.0 & 38.0 & 37.0 & 33.0 & 25.0 & 13.0 & 1.0 \\

0.04 &
0.339 & 0.332 & 0.335 & 0.339 & 0.350 & 0.374 & 0.424 & 0.502 & 0.584 &
43.0 & 44.0 & 43.0 & 43.0 & 41.0 & 37.0 & 28.0 & 15.0 & 1.0 \\

0.05 &
0.342 & 0.341 & 0.343 & 0.346 & 0.354 & 0.374 & 0.419 & 0.497 & 0.587 &
42.0 & 42.0 & 42.0 & 42.0 & 40.0 & 37.0 & 29.0 & 16.0 & 1.0 \\

0.06 &
0.375 & 0.366 & 0.366 & 0.367 & 0.372 & 0.387 & 0.425 & 0.499 & 0.592 &
37.0 & 38.0 & 38.0 & 38.0 & 37.0 & 35.0 & 28.0 & 16.0 & 0.0 \\

0.07 &
\bf 0.429 & 0.394 & 0.394 & 0.395 & 0.398 & 0.409 & 0.440 & 0.507 & 0.600 &\bf 
28.0 & 33.0 & 33.0 & 33.0 & 33.0 & 31.0 & 26.0 & 14.0 & -0.0 \\

0.08 &
\bf 1.480 & \bf 1.477 & \bf 1.477 & \bf 1.476 & 0.430 & 0.438 & 0.463 & 0.523 & 0.610 &
\bf -150.0 & \bf -149.0 &\bf  -149.0 &\bf  -149.0 &27.0 & 26.0 & 22.0 & 12.0 & -0.0 \\

0.09 &
\bf 2.273 & \bf 2.187 & \bf 1.491 & \bf 1.490 & \bf 1.488 & \bf 1.490 &\bf  1.493 & 0.548 & 0.624 &
\bf -284.0 & \bf -269.0 & \bf -152.0 &\bf  -152.0 &\bf  -151.0 & \bf -152.0 &\bf  -152.0 & 7.0 & -0.1 \\
\bottomrule
\end{tabular}}
\end{table*}

Firstly, we evaluate the recovery capability of the sparse regularized projection approximation (RPMA-S) model. We consider a ground-truth projection matrix \(X^*\in\mathbb{R}^{40\times 40}\) with rank \(K=4\). Specifically, \(X^*\) consists of four diagonal blocks of size \(10\times10\), where each entry within a block is equal to \(1/10\), and all off-block entries are zero. To simulate noisy observations, we construct
\[
X= X^* + E,
\]
where \(E\) is a sparse noise matrix. For each entry \((i,j)\), we independently select it with probability \(0.5\); if selected, the entry is assigned a random value drawn uniformly from \([-1,1]\), and otherwise it is set to zero. This procedure produces a highly corrupted observation matrix while preserving the underlying low-rank projection structure.

The solutions obtained by the spectral projection estimator and the regularized projection approximation model with the Huber penalty are presented in Figure~\ref{solutions_compare}. As can be seen, the spectral projection estimator produces numerous small nonzero entries in the off-diagonal blocks, reflecting the influence of noise on the recovered projection matrix. In contrast, the Huber-regularized model yields a substantially cleaner solution, with most off-diagonal entries effectively suppressed. Consequently, the recovered projection matrix more closely resembles the underlying block-diagonal structure, demonstrating the ability of the regularization term to mitigate noise and enhance recovery accuracy.

%In addition, we test the performance our model under different setting of the regularization parameter $\lambda$ and the Huber loss parameter $\delta$. The performance for different setting of parameters is shown in Figure~\ref{improve}. We can conclude that the improvement is obvious and the relative recovery of RPMA can up to 42.8\% under the proper setting of parameters. This obeservation is coherent with our intuiation and validate the effectiveness of our model.

In addition, we investigate the sensitivity of RPMA-S to the regularization parameter $\lambda$ and the Huber parameter $\delta$. The recovery performance under different parameter settings is reported in Figure~\ref{improve1} and Figure~\ref{improve2}. We observe that RPMA-S consistently outperforms the spectral projection estimator across a broad range of parameter choices. In particular, with appropriately selected values of $\lambda$ and $\delta$, the relative improvement in recovery accuracy reaches up to $44.0\%$. This significant gain confirms the advantage of incorporating sparsity-promoting regularization into the projection approximation framework and provides strong empirical evidence for the effectiveness of the proposed model.

Finally, to validate the global convergence property of the Riemannian Gradient method, we initialize \(X= UU^T\) from a completely different starting point by choosing the first \(K\) canonical basis vectors of \(\mathbb{R}^n\), namely,
\[
U=[e_1,e_2,\ldots,e_K],
\]
instead of using the eigenvectors associated with \(X\). This initialization corresponds to a substantially different \(K\)-dimensional subspace and therefore provides a stringent test of the algorithm's global convergence behavior. The recovery error and the improvement for the new initialization is shown in Table~\ref{tab:lambda_delta_canonical_basis}. We can conclude that:

\begin{enumerate}
\item When $\lambda$ is small ($\lambda<0.05$), Table~\ref{tab:lambda_delta_canonical_basis} demonstrate almost the same results compared with Figure~\ref{improve1} and Figure~\ref{improve2}, which validate our global convergence property in Theorem~\ref{thm:stability}.
\item The sparse projection approximation model performs well for relatively small values of $\delta$, and the improvement in approximation quality, measured in the Frobenius norm, can reach up to $44.0\%$.
It is worth noting that as $\delta \rightarrow 0^+$, the Huber loss converges uniformly to the $\ell_1$ norm, thereby promoting sparsity in the resulting solution.

\item As $\lambda$ increases, the curvature induced by the regularization term eventually dominates the spectral gap. Consequently, for $\lambda=0.08$ and $0.09$, the Cayley--SMW Riemannian Gradient method fails to recover the global optimum when $\delta$ is small. However, this phenomenon can be alleviated by increasing $\delta$; for instance, the global optimum is recovered when $\delta>0.004$ for $\lambda=0.08$ and when $\delta>0.064$ for $\lambda=0.09$. This observation is consistent with Theorem~\ref{thm:stability}, where the sufficient condition for preserving the global minimizer only requires $\lambda \leq c\eta_K/\ell = c\eta_K\delta$.
\end{enumerate}

\subsection{Experiments on Real-World Datasets}
\label{subsec:real_data}

\begin{table*}[t!]
\centering
\caption{The objectives and constraints of the baseline algorithms.}
\label{tab:related_algorithms}

\resizebox{\linewidth}{!}{
\begin{tabular}{c|c|c|c|c|c}
\hline\hline
& SDP-1
& SDP-2
& SDP-U
& Spectral Proj.
& SLSA
\\
\hline

Objectives
&
$\displaystyle
\max_X\ \langle A,X\rangle
$
&
$\displaystyle
\max_X\ \langle A,X\rangle
$
&
$\displaystyle
\max_X\ \langle A,X\rangle
$
&
$\displaystyle
\max_X\ \langle A,X\rangle
$
&
$\displaystyle
\min_{X,U}\
\|X-A\|_F^2
+\theta\|X-UU^T\|_F^2
$
\\
\hline

Constraints
&
$
\begin{aligned}
&X\mathbf{1}_n
 =\frac{n}{K}\mathbf{1}_n,
\qquad X\succeq 0,
\\
&\operatorname{diag}(X)=\mathbf{1}_n,
\qquad X\geq 0
\end{aligned}
$
&
$
\begin{aligned}
&\langle X,E_n\rangle=\frac{n^2}{K},
\qquad X\succeq 0,
\\
&\operatorname{tr}(X)=n,
\qquad 0\leq X\leq 1
\end{aligned}
$
&
$
\begin{aligned}
&X\mathbf{1}_n=\mathbf{1}_n,
\qquad \operatorname{tr}(X)=K,
\\
&X\succeq 0,
\qquad X\geq 0
\end{aligned}
$
&
$
X\in\mathcal{P}_K
$
&
$
\begin{aligned}
&U^TU=I_K,
\\
&\|X_{\mathrm{off}}\|_0\leq\eta
\end{aligned}
$
\\
\hline\hline
\end{tabular}
}
\end{table*}

The preceding synthetic experiment demonstrates that the proposed
regularized projection matrix approximation can recover a cleaner
projection structure from noisy observations. We next evaluate the
clustering performance of the two proposed models, RPMA-S and
RPMA-NS, on real-world image data. Compared with the SDP and
spectral-projection baselines, RPMA-S directly estimates a rank-$K$
projection matrix and augments the projection-fitting objective with
an entrywise Huber penalty, which promotes a sparse and noise-robust
projection structure. RPMA-NS further extends RPMA-S by incorporating
a one-sided quadratic penalty on negative entries and a row-sum
penalty,
thereby encouraging entrywise nonnegativity and approximate
row-stochasticity in the recovered projection matrix.

For comparison, SDP-1~\cite{AMINI2018} and SDP-2~\cite{AMINI2018} are used as the SDP baselines in the
balanced-sample experiments. Since these formulations rely on
equal-cluster-size assumptions, they cannot be directly applied when
the class sizes are heterogeneous. Following the unequal-cluster
treatment of SDP-based $K$-means clustering in
\cite{zhuang2022sketch}, we construct a size-flexible variant, denoted
by SDP-U, by removing the equal-capacity restriction while retaining
the positive-semidefinite, entrywise nonnegative, trace, and row-sum
constraints of the standard $K$-means SDP. SDP-U is therefore used as
the SDP baseline in the imbalanced-sample experiments. Spectral
projection~\cite{belkin2003laplacian,vonLuxburg2007} and SLSA~\cite{zhang2022graph} are included in both settings. The objectives and
constraints of these baseline methods are summarized in
Table~\ref{tab:related_algorithms}.

\begin{table}[!t]
    \centering
    \caption{Summary of the real-world datasets.}
    \label{tab:real_datasets}
    \small
    \renewcommand{\arraystretch}{1.12}

    \begin{tabular*}{\columnwidth}{
        @{\extracolsep{\fill}}lccc@{}
    }
        \toprule
        Dataset
        & Classes
        & Samples/Class
        & Total Samples \\
        \midrule

        COIL20
        & 20
        & 72
        & 1,440 \\

        AT\&T Faces
        & 40
        & 10
        & 400 \\

        Semeion
        & 10
        & 155
        & 1,550 \\

        DIGIT-10
        & 10
        & 200
        & 2,000 \\

        \bottomrule
    \end{tabular*}
\end{table}

The real-world experiments involve four datasets covering object
images, facial images, and handwritten digits. Their main statistics
are summarized in Table~\ref{tab:real_datasets}.

\emph{COIL20:}
COIL20 contains grayscale images of 20 objects photographed from 72
equally spaced viewing directions, giving 1,440 images in total
\cite{nene1996coil20}. The continuous viewpoint variation produces
substantial within-class appearance changes and makes the dataset
suitable for testing whether a clustering method can preserve object
identity across different viewing angles.

\emph{AT\&T Database of Faces:}
The AT\&T Database of Faces contains 400 images of 40 individuals, with
10 images per subject \cite{samaria1994parameterisation}. The images
include variations in expression, pose, illumination, and facial
details. The small number of observations per class makes the clustering
performance sensitive to the quality of the recovered projection matrix.

\emph{Semeion:}
The Semeion handwritten-digit dataset contains 1,593 samples from digits
0--9, with each digit represented by a $16\times16$ binary image and
therefore a 256-dimensional feature vector \cite{uci1998semeion}. To
obtain a balanced clustering instance, we randomly select 155 samples
from each class, resulting in 1,550 samples in total.

\emph{DIGIT-10:}
DIGIT-10 is constructed from the UCI Multiple Features handwritten-digit
dataset \cite{duin1998multiplefeatures}. We use the pixel-average feature
view, in which each sample is represented by a 240-dimensional vector.
The dataset contains 10 digit classes with 200 samples per class, yielding
2,000 samples in total.

For COIL20, AT\&T Faces, and DIGIT-10, all available samples in the
selected representation are used. For Semeion, the largest equal-sized
subset considered in our experiments is used. In all cases, the
ground-truth class labels are used only to compute ACC, NMI, and ARI
after clustering and are not provided to any clustering method.

\subsubsection{Balanced-Sample Experiments}
\label{subsubsec:balanced_experiments}

We first evaluate the competing methods under the balanced-sample
setting. The experiments use an equal number of samples in every
class, namely 72, 10, 155, and 200 samples per class for COIL20,
AT\&T Faces, Semeion, and DIGIT-10, respectively. Clustering
performance is assessed using clustering accuracy (ACC), normalized
mutual information (NMI), and adjusted Rand index (ARI).

Under the balanced-sample setting, the final labels produced by
RPMA-NS are obtained through capacity-constrained rounding using the
known common cluster size. The ground-truth class labels are used only
for performance evaluation and are not involved in either the
optimization procedure or the label assignment. The quantitative
results are reported in Table~\ref{tab:balanced_results}.

\begin{table}[t!]
    \centering
    \caption{Clustering performance on real-world datasets.}
    \label{tab:balanced_results}
    \footnotesize
    \renewcommand{\arraystretch}{1.05}

    \begin{tabular*}{\columnwidth}{
        @{\extracolsep{\fill}}llccc@{}
    }
        \toprule
        Dataset & Method & ACC & NMI & ARI \\
        \midrule

        COIL20
        & SDP-1
        & \textbf{0.7271}
        & \textbf{0.8118}
        & \textbf{0.6762} \\

        & SDP-2 & 0.6784 & 0.7803 & 0.5966 \\
        & SLSA & 0.6799 & 0.7605 & 0.6031 \\
        & Spectral Projection & 0.6611 & 0.7833 & 0.6159 \\
        & RPMA-S & 0.6875 & 0.7817 & 0.6248 \\
        & RPMA-NS & 0.7028 & 0.7887 & 0.6436 \\

        \midrule

        AT\&T Faces
        & SDP-1 & 0.8225 & 0.9152 & 0.7677 \\
        & SDP-2 & 0.7800 & 0.8839 & 0.6861 \\
        & SLSA & 0.7375 & 0.8522 & 0.3902 \\
        & Spectral Projection & 0.8275 & 0.9144 & 0.7583 \\
        & RPMA-S & 0.8350 & 0.9212 & 0.7812 \\
        & RPMA-NS
        & \textbf{0.8675}
        & \textbf{0.9275}
        & \textbf{0.8169} \\

        \midrule

        Semeion
        & SDP-1 & 0.4542 & 0.4461 & 0.2997 \\
        & SDP-2 & 0.4852 & 0.4527 & 0.3135 \\
        & SLSA & 0.4510 & 0.4987 & 0.1560 \\
        & Spectral Projection & 0.5110 & 0.4679 & 0.3243 \\
        & RPMA-S & 0.5141 & 0.4657 & 0.3247 \\
        & RPMA-NS
        & \textbf{0.5651}
        & \textbf{0.5127}
        & \textbf{0.4003} \\

        \midrule

        DIGIT-10
        & SDP-1
        & \textbf{0.7590}
        & \textbf{0.7716}
        & \textbf{0.6594} \\

        & SDP-2 & 0.6560 & 0.6475 & 0.5353 \\
        & SLSA & 0.6465 & 0.7418 & 0.5458 \\
        & Spectral Projection & 0.6630 & 0.6370 & 0.5274 \\
        & RPMA-S & 0.7405 & 0.6664 & 0.5782 \\
        & RPMA-NS & 0.7145 & 0.6264 & 0.5318 \\

        \bottomrule
    \end{tabular*}
\end{table}

% Insert Table~\ref{tab:balanced_results} here.

Overall, the RPMA-S family exhibits consistently competitive performance
across the four datasets. RPMA-NS achieves the best results on
AT\&T Faces and Semeion and remains the second-best method on COIL20
for all three evaluation criteria. These results indicate that
incorporating nonnegative and stochastic structure can improve the
quality of the recovered clustering projection, particularly when
the affinity matrix is compatible with the assumed cluster structure.

The basic RPMA-S model also performs strongly. It substantially
outperforms classical spectral projection on several datasets and
remains especially competitive on DIGIT-10, where it achieves better
ACC, NMI, and ARI than the spectral baseline. The fact that RPMA-S
outperforms RPMA-NS on this dataset also suggests that the additional
structural constraints are not uniformly beneficial for every feature
representation and may become restrictive when the input affinity
deviates from an ideal block structure.

Taken together, these results support the central premise of the
proposed framework: directly regularizing a rank-$K$ projection matrix
provides an effective alternative to conventional spectral embedding.
RPMA-S offers a robust general formulation, while RPMA-NS can yield
further improvements when nonnegativity, stochasticity, and balanced
cluster structure are appropriate for the data.

\begin{table*}[!t]
    \centering
    \caption{Clustering performance under the imbalanced-sample setting.}
    \label{tab:imbalanced_results}
    \small
    \setlength{\tabcolsep}{5.2pt}
    \renewcommand{\arraystretch}{1.08}

    \begin{tabular}{@{}llccccc@{}}
        \toprule
        Dataset
        & Sampling ratio
        & SDP-U
        & RPMA-S
        & RPMA-NS
        & Spectral Proj.
        & SLSA \\
        \midrule

        \multicolumn{7}{c}{\textit{(a) ACC}} \\
        \midrule

        AT\&T Faces
        & 20\%
        & \textbf{0.7250}
        & 0.6875
        & 0.6875
        & 0.6875
        & 0.6500 \\

        AT\&T Faces
        & 40\%
        & \textbf{0.7125}
        & 0.6938
        & 0.7000
        & 0.6813
        & 0.6813 \\

        AT\&T Faces
        & 60\%
        & \textbf{0.7667}
        & 0.7542
        & 0.7625
        & 0.7333
        & 0.7292 \\

        AT\&T Faces
        & 80\%
        & 0.8156
        & 0.8156
        & \textbf{0.8344}
        & 0.8156
        & 0.7750 \\

        COIL20
        & 20\%
        & 0.6563
        & \textbf{0.7153}
        & 0.7118
        & 0.6458
        & 0.5451 \\

        COIL20
        & 40\%
        & 0.4757
        & 0.6667
        & \textbf{0.7205}
        & 0.6267
        & 0.5278 \\

        COIL20
        & 60\%
        & 0.2662
        & 0.6597
        & 0.6759
        & \textbf{0.6852}
        & 0.4942 \\

        COIL20
        & 80\%
        & 0.0677
        & 0.6866
        & \textbf{0.6953}
        & 0.6814
        & 0.4644 \\

        \midrule
        \multicolumn{7}{c}{\textit{(b) NMI}} \\
        \midrule

        AT\&T Faces
        & 20\%
        & 0.9138
        & \textbf{0.9194}
        & 0.9066
        & 0.8960
        & 0.8761 \\

        AT\&T Faces
        & 40\%
        & \textbf{0.9081}
        & 0.8858
        & 0.8783
        & 0.8899
        & 0.8742 \\

        AT\&T Faces
        & 60\%
        & \textbf{0.9068}
        & 0.9040
        & 0.8991
        & 0.8974
        & 0.8770 \\

        AT\&T Faces
        & 80\%
        & 0.9156
        & \textbf{0.9281}
        & 0.9175
        & 0.9235
        & 0.8935 \\

        COIL20
        & 20\%
        & 0.8035
        & 0.8355
        & \textbf{0.8446}
        & 0.7963
        & 0.7536 \\

        COIL20
        & 40\%
        & 0.6672
        & 0.8004
        & \textbf{0.8497}
        & 0.7991
        & 0.7249 \\

        COIL20
        & 60\%
        & 0.3596
        & 0.8296
        & \textbf{0.8349}
        & 0.8090
        & 0.6703 \\

        COIL20
        & 80\%
        & 0.0304
        & 0.7967
        & \textbf{0.8130}
        & 0.8033
        & 0.6061 \\

        \midrule
        \multicolumn{7}{c}{\textit{(c) ARI}} \\
        \midrule

        AT\&T Faces
        & 20\%
        & 0.5274
        & \textbf{0.5721}
        & 0.4877
        & 0.4491
        & 0.3719 \\

        AT\&T Faces
        & 40\%
        & \textbf{0.6832}
        & 0.6402
        & 0.6374
        & 0.6122
        & 0.5806 \\

        AT\&T Faces
        & 60\%
        & \textbf{0.7423}
        & 0.7140
        & 0.7101
        & 0.6843
        & 0.6670 \\

        AT\&T Faces
        & 80\%
        & 0.7722
        & \textbf{0.7992}
        & 0.7960
        & 0.7941
        & 0.7298 \\

        COIL20
        & 20\%
        & 0.7249
        & 0.7708
        & \textbf{0.7751}
        & 0.7120
        & 0.4998 \\

        COIL20
        & 40\%
        & 0.3023
        & 0.6485
        & \textbf{0.7248}
        & 0.6315
        & 0.3539 \\

        COIL20
        & 60\%
        & 0.0658
        & 0.6776
        & \textbf{0.6999}
        & 0.6575
        & 0.2079 \\

        COIL20
        & 80\%
        & 0.0503
        & 0.6545
        & \textbf{0.6762}
        & 0.6584
        & 0.1351 \\

        \bottomrule
    \end{tabular}
\end{table*}

\subsubsection{Imbalanced-Sample Experiments}
\label{subsubsec:imbalanced_experiments}

We next investigate the robustness of the competing methods when the
class sizes are heterogeneous. For each dataset, four sampling ratios,
$r\in\{20\%,40\%,60\%,80\%\}$, are considered. Let $N$ and $K$
denote the total number of available samples and the number of classes,
respectively. At sampling ratio $r$, the target subset size is
\begin{equation}\notag
    M_r=\left\lfloor rN\right\rfloor .
\end{equation}
A class-proportion vector is generated according to
\begin{equation}\notag
    (\pi_1,\ldots,\pi_K)
    \sim
    \operatorname{Dirichlet}
    \left(\alpha_{\mathrm{imb}}\mathbf{1}_K\right),
\end{equation}
and the resulting proportions are converted into integer class sizes
subject to
\begin{equation}\notag
    m_{\min}\leq m_k\leq N_k,
    \qquad
    \sum_{k=1}^{K}m_k=M_r,
\end{equation}
where $N_k$ is the number of available samples in class $k$. For each
sampling ratio, all methods are evaluated using exactly the same
sampled subset and affinity matrix.

Since the cluster capacities are unknown under the imbalanced
setting, capacity-constrained balanced rounding is no longer
applicable. Ordinary K-means is therefore used to discretize the
continuous representations produced by RPMA-S, RPMA-NS, spectral
projection, and SLSA. Consequently, the performance of RPMA-NS in
this experiment does not rely on the equal-size prior or on any
ground-truth cluster-capacity information.The quantitative results are reported in
Table~\ref{tab:imbalanced_results}.

Although SDP-1 is highly competitive in the balanced experiments,
the corresponding unbalanced formulation does not exhibit comparable
stability under heterogeneous class sizes. This behavior is especially
pronounced on COIL20. As the sampling ratio increases from $20\%$ to
$80\%$, the ACC of SDP-U decreases from $0.6563$ to $0.0677$, while
its NMI decreases from $0.8035$ to $0.0304$ and its ARI decreases from
$0.7249$ to $0.0503$. Hence, retaining more samples does not lead to
an improvement for SDP-U; instead, its performance deteriorates
sharply as the imbalanced instance becomes larger. This result
indicates that the modified SDP formulation remains sensitive to the
interaction between heterogeneous cluster sizes and its structural
constraints.

In contrast, RPMA-S and RPMA-NS maintain substantially more stable
performance over the four sampling ratios. On COIL20, RPMA-NS attains
the highest NMI and ARI at every sampling level, while its ACC remains
within a relatively narrow range from $0.6759$ to $0.7205$. RPMA-S also
preserves competitive ACC, NMI, and ARI values as the number of
retained samples increases. On AT\&T Faces, both variants remain
competitive with SDP-U and become particularly effective at the
largest sampling ratio. These observations suggest that regularized
projection recovery is less dependent on an accurately specified
cluster-size model and is therefore better suited to heterogeneous
sample distributions.

\begin{figure*}[!t]
    \centering

    % ==================== AT&T Faces ====================

    \subfloat[AT\&T Faces, 20\%]
    {
        \includegraphics[width=0.235\linewidth]
        {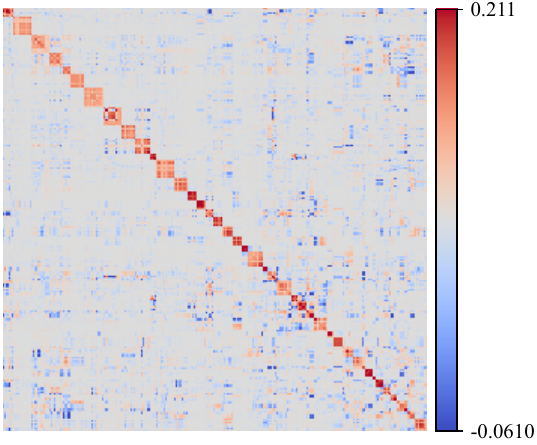}
        \label{fig:att_faces_rpma_20}
    }
    \hfill
    \subfloat[AT\&T Faces, 40\%]
    {
        \includegraphics[width=0.235\linewidth]
        {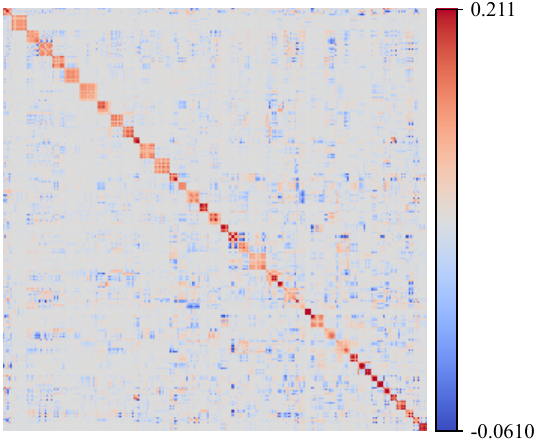}
        \label{fig:att_faces_rpma_40}
    }
    \hfill
    \subfloat[AT\&T Faces, 60\%]
    {
        \includegraphics[width=0.235\linewidth]
        {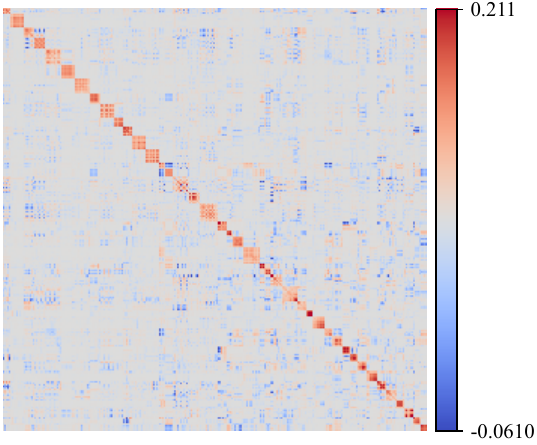}
        \label{fig:att_faces_rpma_60}
    }
    \hfill
    \subfloat[AT\&T Faces, 80\%]
    {
        \includegraphics[width=0.235\linewidth]
        {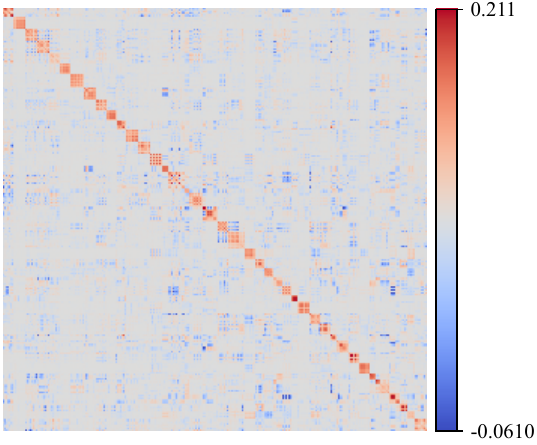}
        \label{fig:att_faces_rpma_80}
    }

    \vspace{1.5mm}

    % ====================== COIL20 ======================

    \subfloat[COIL20, 20\%]
    {
        \includegraphics[width=0.235\linewidth]
        {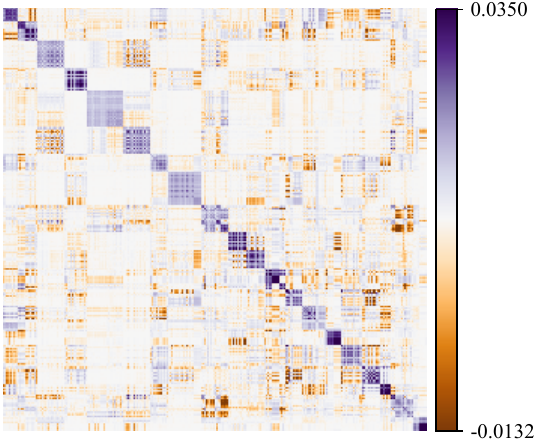}
        \label{fig:coil20_rpma_20}
    }
    \hfill
    \subfloat[COIL20, 40\%]
    {
        \includegraphics[width=0.235\linewidth]
        {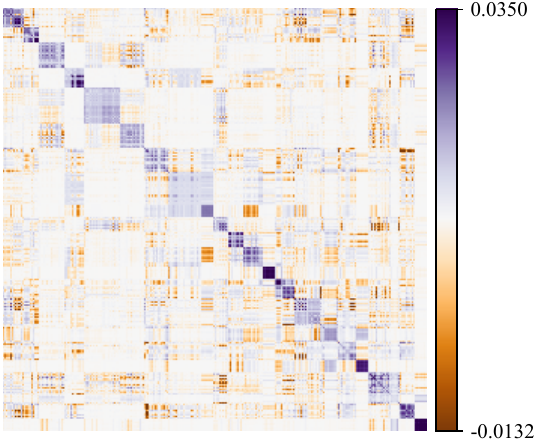}
        \label{fig:coil20_rpma_40}
    }
    \hfill
    \subfloat[COIL20, 60\%]
    {
        \includegraphics[width=0.235\linewidth]
        {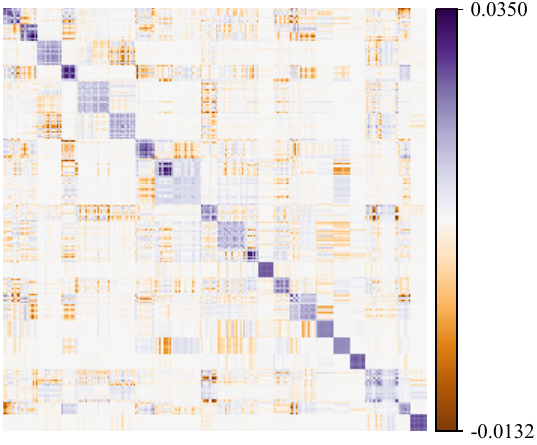}
        \label{fig:coil20_rpma_60}
    }
    \hfill
    \subfloat[COIL20, 80\%]
    {
        \includegraphics[width=0.235\linewidth]
        {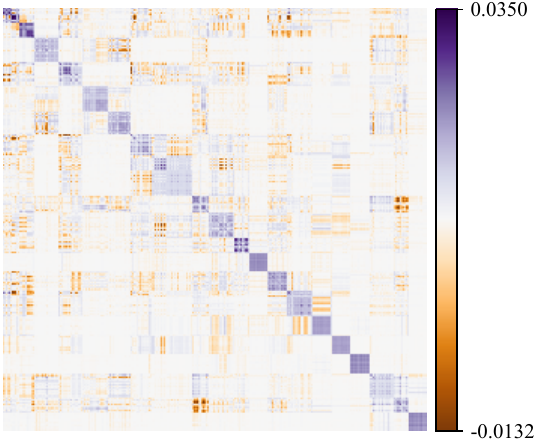}
        \label{fig:coil20_rpma_80}
    }

    \caption{Reordered RPMA-S projection matrices under different
    imbalanced sampling ratios. The first and second rows correspond
    to AT\&T Faces and COIL20, respectively. From left to right, the
    sampling ratios are 20\%, 40\%, 60\%, and 80\%. The four panels
    associated with the same dataset share a common color scale, while
    the color scales are normalized independently across the two
    datasets.}
    \label{fig:imbalanced_rpma_heatmaps}
\end{figure*}

Figure~\ref{fig:imbalanced_rpma_heatmaps} provides a qualitative view
of the projection structures recovered by RPMA-S. Across the four
sampling levels, the matrices retain a visible concentration of
large-magnitude entries around the reordered diagonal and preserve
recognizable blockwise patterns. In particular, the projection
structure does not collapse as the number of retained samples
increases. This visual evidence is consistent with the quantitative
results and further illustrates the stability of RPMA--S under unequal
class distributions.

\subsection{Discussion}
\label{subsubsec:real_data_discussion}

The balanced and imbalanced experiments reveal an important distinction
between the competing formulations. Under balanced class sizes, SDP-1
can exploit its equal-capacity assumption effectively and achieves
strong performance on several datasets. However, this advantage does
not transfer consistently to the unequal-size setting. Although SDP-U
relaxes the original size constraint, its sharp deterioration on
COIL20 demonstrates that the modified formulation can remain highly
sensitive to the class-size configuration and to the scale of the
sampled problem.

RPMA-S does not require an exact equal-capacity constraint. Instead, it
directly estimates a regularized rank-$K$ projection matrix whose
dominant structure encodes the latent partition. This formulation
provides competitive performance in the balanced experiments while
remaining stable when the class sizes become heterogeneous. RPMA-NS
further promotes nonnegativity and row-stochasticity, which are
consistent with the properties of an ideal clustering projection.
Its consistently strong NMI and ARI under the imbalanced COIL20
experiments indicate that these structural properties help suppress
spurious inter-cluster connections without relying on known cluster
capacities.

Overall, the results demonstrate that the main advantage of the RPMA-S
family lies not only in its competitive clustering accuracy, but also
in its robustness across substantially different sampling regimes.
RPMA-S provides a flexible general-purpose formulation, whereas RPMA-NS
offers additional structural regularization when the affinity matrix
is compatible with a nonnegative stochastic projection. Their stable
behavior under unequal class sizes provides a clear advantage over
methods whose performance depends more strongly on balanced-size
assumptions.

\section{Conclusion}
In this paper, we investigated the sparsed regularized projection matrix approximation (RPMA) framework and developed efficient optimization algorithms with reduced computational complexity. We established sufficient conditions for local optimality and characterized the stability of the regularized leading eigenspace and the critical-point landscape. In particular, our analysis reveals how the regularization parameter influences the geometry of the optimization landscape and the persistence of critical points under perturbation.

Several promising directions remain for future research. First, it would be worthwhile to investigate a broader class of regularization terms and their impact on clustering accuracy, robustness, and interpretability. Second, developing adaptive strategies for selecting regularization parameters in a data-driven manner is an important practical problem. Finally, it is of considerable theoretical interest to extend the current analysis beyond the sufficient conditions established in this paper, including understanding the behavior of RPMA when these conditions are violated and deriving sharper stability and landscape characterizations under weaker assumptions.

\bibliographystyle{IEEEtran}
\bibliography{Cmm}

\begin{IEEEbiographynophoto}{Zhuan Liang}
is currently a graduate student in the Department of Statistics, Faculty of Arts and Sciences, Beijing Normal University, Zhuhai.
He received the B.S. degree in Mathematics and Statistics from Shandong University. His current research interests include diffusion models and matrix computations.
\end{IEEEbiographynophoto}

\begin{IEEEbiographynophoto}{Zheng Zhai}
is currently an Associate Professor with the Department of Statistics, Faculty of Arts and Sciences, Beijing Normal University, Zhuhai. He received the Ph.D. degree in Computational Mathematics from Zhejiang University. He was a Postdoctoral Researcher with the National University of Singapore and the University of Toronto. His current research interests include diffusion models, matrix computations, optimization on manifolds, and manifold learning.
\end{IEEEbiographynophoto}

\end{document}